\documentclass[12pt]{article} 
\usepackage{fullpage} 

\usepackage{amsmath,amsthm,amssymb,mathtools,amsfonts,bbm}
\usepackage{graphicx,caption,color,comment,hyperref,natbib,boldline,color,booktabs}
\usepackage[boxed]{algorithm2e}

\theoremstyle{definition}

\newtheorem{theorem}{Theorem}[section]
\newtheorem{lemma}{Lemma}[section]
\newtheorem{corollary}{Corollary}[section]

\DeclareMathOperator{\R}{\mathbb{R}}

\newcommand{\bis}{\textsc{Phased-SE}}
\newcommand{\bie}{\textsc{UCB-Revisited$+$}}
\newcommand{\ucbr}{\textsc{UCB-Revisited}}
\newcommand{\ucb}{\textsc{UCB1}}
\newcommand{\moss}{\textsc{MOSS}}
\newcommand{\aae}{\textsc{SE}}
\newcommand{\odaaf}{ODAAF}
\newcommand{\narms}{K}
\newcommand{\arms}{\mathcal{K}}
\newcommand{\indicator}{\mathbb{I}}
\newcommand{\armsm}{\mathcal{K}_m}
\newcommand{\armsmp}{\mathcal{K}_{m+1}}
\newcommand{\estmeanmj}{\overline{X}_{m,j}}
\newcommand{\dtm}{\tilde{\Delta}_m}
\newcommand{\dmax}{d_{\max}}
\newcommand{\expd}{\mathbb{E}[d]}
\newcommand{\bucketsm}{\mathcal{B}_{m}}
\newcommand{\bucketl}{B_l}
\newcommand{\bucketlnew}{B_{l}^{'}}

\begin{document}
\title{Bandits with Temporal Stochastic Constraints}
\author{Priyank Agrawal${}^{1}$ and Theja Tulabandhula${}^{2}$\\
${}^{1}$Indian Institute of Science\\
${}^{2}$University of Illinois at Chicago}

\maketitle

\begin{abstract}
We study the effect of impairment on stochastic multi-armed bandits and develop new ways to mitigate it. Impairment effect is the phenomena where an agent only accrues reward for an action if they have played it at  least a few times in the recent past. It is practically motivated by repetition and recency effects in domains such as advertising (here consumer behavior may require repeat actions by advertisers) and vocational training (here actions are complex skills that can only be mastered with repetition to get a payoff). Impairment can be naturally modelled as a temporal constraint on the strategy space, and we provide two novel algorithms that achieve sublinear regret, each working with different assumptions on the impairment effect. We introduce a new notion called bucketing in our algorithm design, and show how it can effectively address impairment as well as a broader class of temporal constraints. Our regret bounds explicitly capture the cost of impairment and show that it scales (sub-)linearly with the degree of impairment. Our work complements recent work on modeling delays and corruptions, and we provide experimental evidence supporting our claims.
\end{abstract}

\section{Introduction}\label{sec:introduction}

The multi-armed bandit (MAB) problem captures exploration-exploitation tradeoff in sequential decision making. In these problems, the algorithm is tasked with learning which arm (action) is the best while simultaneously exploiting its information to obtain rewards. It has been extensively studied in various disciplines such as operations research, statistics and computer science and various solution strategies have been used successfully in domains such as advertising, recommendation systems and clinical trials. In a typical problem, stochastic reward distributions are associated with each arm in a set of arms $\arms$. When the algorithm plays arm $j$, it immediately receives a reward with which it can learn, and also suffers a regret, which is the difference between the obtained reward versus the reward that it could have obtained, had it played the best arm.


In one of the key applications of MAB, viz., advertising, there has been a large body of work that models consumer behavior~\citep{hawkins2009consumer,solomon2014consumer}. One such effect that is relatively well studied is the repetition effect~\citep{machleit1988emotional,campbell2003brand}, under which, an advertiser's payoff (for instance, click through rate) depends on how frequently they have presented the same ad to the same audience in the recent past. If the advertiser presents a specific ad sporadically, then the aggregated payoff is much lower. If this ad is the best among a collection of ads, then the advertiser will not be able to deduce this from their observations. Further, different ads may need different levels of repetition to obtain payoffs, and this may not be known a priori to the advertiser. This phenomenon also translates to recommendations, such as for products and movies, where repeated display of item(s) can cause positive reinforcement to build over time, culminating in a conversion. And since conversions depend on the past recommendations, they interfere with learning the true underlying (mean) payoffs of different recommendations. In the domain of skill acquisition~\citep{dekeyser2007skill}, especially those which are complex~\citep{bosse2015benefit}, a learner may have to repeatedly try each action several times to advance or acquire a reward. Further, they may have to repeat these actions frequently enough so as to not lose the acquired skill~\citep{kang2016spaced}. Under this effect, their ability to learn the best action (or set of actions) may be severely impaired.

Motivated by the above discussion, we define a new class of problems, which we call \emph{bandit learning with stochastic impairment}, to address the need for repetitions. The defining characteristic here is that a learning algorithm can only accrue rewards if it has played the same arm enough number of times in the recent past. The amount by which the algorithm needs to replay the same arm is a measure of impairment (see Section~\ref{sec:problem_definition} for a formal treatment). When the arm has been played enough times in the past, then the reward for playing that arm at the current time is instantaneous. In other words, the reward accrued by the algorithm is a function of the current reward generated and the sequence of arms played in the recent past. A diagram illustrating this is shown in Figure~\ref{fig:1}. Our formulation is an example of combining empirically validated behavioral models with sequential decision making, and a few other examples are discussed towards the end of this section.

\begin{figure}
\centering
\includegraphics[width=.5\textwidth]{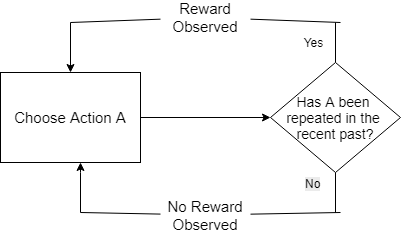}
\caption{Bandit learning with stochastic impairment.}
\end{figure}\label{fig:1}



In the presence of impairment effect, usual MAB algorithms such as \ucb, \aae, \moss{} or Thompson Sampling are ineffective because one cannot directly control the number of times an arm is played in a given time window. In fact, if we have an instance in which a couple of arms have almost equal mean payoffs, then the aforementioned algorithms may switch between these very frequently (see Section~\ref{sec:experiments}), potentially causing linear regret. To address these issues, we develop two algorithms, \bis{} and \bie, which expand on the phase-based algorithmic template to mitigate impairments in reward accrual. Central to the \bis{} algorithm is the notion of \emph{bucketing}. The algorithm divides the time horizon into several phases and eliminates arms between phases. In addition, within each phase, it also partitions the existing arms into buckets. This partitioning ensures that arms in each bucket are repeated often. It also lets us use an arm elimination strategy based on \aae{}~\citep{even2006action} to incentivise exploration-exploitation, and facilitates arm elimination  within phases. As we show in our regret upper bound derivations, having buckets has a non-trivial effect on the leading order terms, and also smoothly degrades to a phase-based algorithm's performance (such as that of \bie{} in Section~\ref{sec:expectation_known}) depending on the instance. Our bucketing strategy is quite general and can be used not only in settings with endogenous effects such as the above, but also in settings with exogenous effects. For instance, in a variant of the \emph{distributed bandit learning} setting~\citep{hillel2013distributed,buccapatnam2015information,gamarnik2018delay}, one may have to partition the set of arms such that each node can only play its local set. These nodes then have to communicate with each other and minimize regret. Our bucketing strategy works out of the box for this, where communication between nodes happens between phases, and within a phase, nodes run the arm elimination strategy discussed above. For simplicity of exposition, we focus on impairment as the main setting in the rest of the paper.



A unifying view of our contribution in the stochastic MAB landscape, and recent works such as~\citep{pike2018bandits} and ~\citep{lykouris2018stochastic} is as follows. In all three works, one can assume that there is a intermediate function that allows an MAB algorithm to accrue a transformation of the rewards. \cite{pike2018bandits} address a challenging setting where the algorithm does not accrue a reward for an arm played immediately. Instead it accrues a linear combination of components of rewards of a stochastic number of arms played in the recent past. \cite{lykouris2018stochastic} address the setting where the algorithm accrues the stochastic reward most of the time, but an adversary can corrupt some of the rewards arbitrarily, and the algorithm accrues this incorrect reward as well. In our case, we accrue rewards that are modulated by the sequence of actions that our algorithm took. These transformations from true rewards generated to accrued rewards that depend on various factors make the learning problem more challenging. 

Our setting is very different than the corruption effect studied in~\citep{gajane2017corrupt} and~\citep{lykouris2018stochastic}. \cite{gajane2017corrupt} exploit the corruptions in the reward process as a means to preserve privacy. On the other hand, \cite{lykouris2018stochastic} consider settings with arbitrary exogenous corruptions and propose a randomized algorithm (based on \aae) that achieves smooth degradation as corruption level increases, and recovers the vanilla MAB guarantee when there is no corruption. Unlike their setting, impairment effect is endogenous and correlates past actions with the current reward. While the amount of impairment in our setting (if one uses a phase based strategy) would be equivalent to O($d_{\max} K\log T$) amount of corruption (see Section~\ref{sec:support_known}), we cannot reuse their techniques because of endogeneity. Note that corruption can be unbounded if one uses a regular MAB algorithm such as \ucb.

A stream of works have look at the impact of delayed rewards on learning and regret. In these settings, an algorithm may accrue a reward for its action after a (random) duration of time, possibly mixed with other delayed rewards. For instance, ~\cite{perchet2016batched} consider a setting in which actions are played in batches and rewards are only available at some fixed time points. Works such as~\citep{joulani2013online,cesa2018nonstochastic} have quantified the impact of delay on regret. In particular, ~\cite{joulani2013online} provide a recipe to use any regular MAB algorithm in this setting and show that delay causes an additive regret penalty. In~\citep{cesa2018nonstochastic} for adversarial bandits and~\citep{pike2018bandits} for stochastic bandits, the authors provide regret guarantees for a much weaker setting where rewards can get mixed up and may partially accrue over time. That is, components of rewards due to multiple previous actions may appear collectively. In contrast to the delay effect, which is affecting the future accrual of rewards, the impairment effect can be viewed as being caused by the trajectory of past actions. In both settings, using a phase based algorithmic strategy can mitigate the corresponding effect: for delay, repeating an arm can allow the algorithm to map the reward to the right arm that generated it; for impairment, repeating an arm can allow the algorithm to accrue instantaneous rewards as the arm has been played multiple times in the recent past.


\subsection{Our Results and Techniques}

Owing to the nature of impairment effect, algorithms necessarily have to ensure that the arms still under consideration are played frequently, perhaps in batches. In particular, if one plays an arm for a sufficiently long period of time, then they can ensure that the instantaneous rewards are always accrued (except some at the beginning). Phase based algorithms are a natural choice given these considerations. This family of algorithms date back to ~\citep{agrawal1988asymptotically}, who considered arm switching costs. Our algorithms, \bis{} and \bie, follow this design pattern, wherein the focus is to eliminate arms between stages where each arm is played consecutively for multiple rounds. In particular, \bie{} is a based on \ucbr{}~\cite{auer2010ucb} and works under the setting when the expected impairment effect is known. On the other hand, \bis{} extends \bie{} and incorporates the notion of bucketing to further explore and exploit within phases. It assumes knowledge of an upper bound on the stochastic impairment effect. Naturally, both algorithms also work when the impairment is deterministic.

Following the analysis techniques of~\citep{auer2010ucb} and ~\citep{pike2018bandits}, where the latter uses Freedman's inequality and the  Azuma-Hoeffding inequality based Doob's optimal skipping theorem (see Section~\ref{sec:expectation_known}), we are able to show additive influence of impairment on worst case regret. While our analysis partially overlaps with these two previous works, the random variables related to impairment in our setting do need a qualitatively different treatment. Assume that an algorithm can accrue a reward for pulling arm $j$ at time $t$ if it has been tried at least $d_{t,j}$ times in the past $N$ rounds, where $d_{t,j}$ is a random variable with mean $\expd$ and upper bound $\dmax$ and $N$ is a fixed positive integer. Then, the regret upper bounds of our methods are listed in Table~\ref{tab:introresults}. The impact of impairment is sublinear for \bie{} because of a tighter analysis.

\begin{table}
    \centering
\resizebox{\columnwidth}{!}{
\begin{tabular}{ |c||c|c|  }
\hline
Algorithm\textbackslash Setting   &   Without impairment   &  With impairment\\
\hline
Lower bound    &   $\mathrm{O}(\sqrt{KT})$    & -\\
\ucb     &   $\mathrm{O}(\sqrt{KT\log T })$    & -\\
\moss    &   $\mathrm{O}(\sqrt{KT})$          & -\\
\bis    & $\mathrm{O}(\sqrt{KT\log T })$     & $\mathrm{O}\left(\sqrt{KT\log T} +\sqrt{\frac{K^3\log^3(T)}{T}}\dmax\right)$\\
\bie    & $\mathrm{O}(\sqrt{KT\log T })$     & $ \mathrm{O}\left(\sqrt{KT\log T} + K\sqrt{\log^2 T \expd}\right)$\\
  \hline
\end{tabular}\label{tab:results}
}
\caption{Results summary.}
\label{tab:introresults}
\end{table}


\subsection{Related work}

\cite{shah2018bandit} study bandit problems under a different behavioral effect rooted in microeconomic theory, namely self-reinforcement. In their setting, the algorithm interacts with a sequence of incoming users, who develop preference to the arms that have been played in the past. This positive externality affects the performance of algorithms, where the best arm's performance may get overshadowed by that of a suboptimal arm due to reinforcement. In our case, the frequency of arm plays in the recent past influences the rewards accrued for that arm, which is a different type of externality.

Impact of imperfect memory, or the difficulty of remembering the observed rewards of arms, on an algorithm's performance was studied in~\citep{xu2018reinforcement} and~\citep{gamarnik2018delay}. The models are in continous time and either the rewards are associated with certain arrival and departure processes~\cite{xu2018reinforcement}, or the incoming set of actions are sent to a distributed set of nodes~\cite{gamarnik2018delay}. While the spirit of these is similar to our work, they are not closely related to the discrete time bandit learning setting. Our work is one of the first to address impairment in this case, where impairment can be interpreted in terms of constraints on memory.

Bandits with switching costs have also been studied extensively~\cite{banks1994switching,dekel2014bandits} and consider penalties for switching arms too often. Phase based strategies such as the algorithms that we propose can also be considered as solutions to this problem. Impairment effect imposes a hard constraint on switching arms too frequently, which is approximately equivalent to having large switching costs.

\cite{shivaswamy2012multi} consider the setting where the algorithm has access to previous plays of each arm and focus on warm-starting the exploration-exploitation process. This could be useful in our setting, where one can partition the initial set of arms into buckets such that impairment is mitigated, learn the local best arms in each bucket. Next, these best arms can be combined together (hierarchically agglomerate) and their historical plays can be used to warm-start. A rigorous analysis of this heuristic is a possible future work. In a similar vein of learning with previously collected observational data, ~\cite{swaminathan2015counterfactual} focus on counterfactual risk minimization. While these works work with a one-time use of historical data to improve future payoffs, in our setting we work around the continual impact of historical actions on future payoffs.~\cite{maillard2011adaptive} consider the problem of history-dependent arm selection against possibly adaptive adversaries, which is a much weaker setting compared to ours.

The notion of bucketing incentivises exploration and exploitation within each phase (see Section~\ref{sec:support_known}). And this is an important attribute when the environment and specifications are changing. For instance,~\cite{kleinberg2010regret} introduce the sleeping bandits problem, where the set of available arms changes with time and the algorithm needs to balance exploration and exploitation when new actions are made available. Similar models have also been considered in~\citep{chakrabarti2009mortal,kanade2014learning}. In~\citep{garivier2008upper,garivier2011upper}, the authors assume piece-wise stationarity of the rewards distributions (i.e., the reward distributions abruptly change at certain breakpoints in the time horizon).Under this model, they propose algorithms that estimate mean rewards of arms using only recent history (e.g., using a sliding-window or by discounting). However, none of these overlap with or subsume the impairment effect considered in this work.


\section{Problem Definition}\label{sec:problem_definition}

There are $\narms > 1$ arms in the set $\arms$. Each arm $j \in \arms$ is associated with a reward distribution $\xi_j$, which has support $[0,1]$. The mean reward for arm $j$ is $\mu_j$. $\mu^*$ is the maximum of all $\mu_j$ and corresponds to the arm $j^*$. Define $\Delta_j = \mu^* -\mu_j$. Additionally, let $\epsilon := \underset{j,k \in \arms \setminus\{j^*\} : \Delta_j \leq \Delta_k}{\max}\frac{\Delta_j}{\Delta_k}$. This instance dependent parameter captures how dissimilar the arms are with respect to each other. It also influences the performance of our algorithm in Section~\ref{sec:support_known}. In particular, a smaller value of $\epsilon$ corresponds to an easy instance(i.e., the true means of arms are well separated), whereas a larger value corresponds to a relatively difficult instance (the true means are close to each other). For easy instances, our algorithm's regret guarantee is strictly better than a benchmark (see Section ~\ref{sec:support_known} for more details).

One of the key aspects of this work is a novel modeling of \emph{impairment} that temporally correlates the rewards accrued by any algorithm with its past sequence of actions. Each arm $j$ comes with an i.i.d. stochastic process $\{d_{t,j}\}$ that controls reward accrual in the following way: let $R_{t,j}$ be the reward that would be generated if the $j^{th}$ arm is played at time $t$. Further, let $J_t \in \arms$ represent the arm that is played at time $t$. The reward accrued by the algorithm at time $t$ is not $R_{t,J_t}$ but the following value: 

\begin{equation*}
    X_{t,j} =  R_{t,j} \indicator\left[ \left(\sum_{k= \max(t-N,0)} ^t \indicator[ J_k = j]\right) \geq d_{t,j} \right],
    \label{eq:feedback}
\end{equation*}
where $\indicator[]$ is the indicator function, $N \in \mathbb{N}$ is a instance specific parameter, and for all $j \in \arms$, the random variables $d_{t,j}$ are supported on the set $\{0,1,...,\dmax\}$ with $\dmax \leq N$. For simplicity, we assume that the mean of each random variable $d_{t,j}$ is equal to $\expd$.

Intuitively, the algorithm only accrues reward for playing arm $j$ at time $t$ if it has played the same arm at least $d_{t,j}$ number of times in the past $N$ rounds. Thus, the bandit instance in our work has additional stochastic processes and a parameter $N$ for modeling impairment. 

Given this context, our objective is to design an online algorithm that receives the temporally dependent bandit feedback shown in~(\ref{eq:feedback}) and minimizes expected (pseudo-)regret. That is, for a given time horizon of plays $T$, the algorithm chooses a sequence $\{J_t\}_{t=1}^{T}$ and is scored by:
\begin{equation}\label{eq:regret}
\begin{split}
    R_T &= \max_{ k \in [K]} \mathbb{E}\left[\sum_{t=1}^T X_{t,k} \right] - \mathbb{E}\left[ \sum_{t=1}^{T} X_{t,J_t}\right].
\end{split}
\end{equation}

Unlike the standard setting, the above terms cannot be further simplified because the random variables $\{X_{t,j}\}$ depend on both the instantaneous rewards as well as the historical sequence of arms played. Further, due to impairment, even if one knows the optimal arm $j^*$, they may still receive no rewards for O$(\dmax)$ number of rounds. We assume that the algorithm can differentiate between accruing a reward and and not accruing a reward (i.e., not count the latter as $0$ reward).

Given the above model, one can deduce that methods such as UCB1~\citep{auer2002finite}, MOSS~\citep{audibert2009minimax} or Thompson Sampling~\citep{chapelle2011empirical} will not succeed for any reasonable values of $\expd,\dmax$ and $N$. For instance, the UCB1 algorithm, which is based on the optimism principle, adaptively decreases its optimism over mean rewards such that arms are easily distinguishable from each other. But while doing this, it does not allow for a direct control on how arms switch between rounds. When $N \rightarrow \infty$, the impairment effect fades away completely and the above algorithms will perform well. To control the switching between arms, one could employ algorithms such as \aae~\citep{even2006action} or UCBR~\citep{auer2010ucb}, which either play arms in a predictable round robin fashion or play the same arm consecutively. For instance, when $N = \dmax$, an algorithm such as UCBR is necessary but not sufficient. In this work, we build on the family of \emph{phase based} algorithms, which include UCBR and ODAAF~\citep{pike2018bandits}, and propose a new algorithm that is able to control the arm switching behavior using a novel \emph{bucketed} strategy. This algorithm works under the setting when only $\dmax$ and $N$ are known (see Section~\ref{sec:support_known}). Further, we also show how a minor modification of UCBR/ODAAF is also a very effective regret minimizing strategy. This latter algorithm works when $\expd$ and $N$ are known, and comes with a tighter regret upper bound (see Section~\ref{sec:expectation_known}).

\section{Impairment with Known Support}\label{sec:support_known}

We introduce a new algorithmic innovation (the idea of \emph{bucketing}) that significantly extends the phased-based UCB algorithm (UCBR) ~\cite{auer2010ucb}, which has been extensively studied and extended by a number of works such as~\cite{hillel2013distributed,bui2011committing,buccapatnam2014stochastic,pike2018bandits}. Many optimism based algorithms (for instance, variants of UCB1) tend to switch arms a lot, especially when there are multiple (near-)optimal arms present. In contrast, a phase based structure enables one to control the temporal pattern of arms being played and limit switching. In previous works, this pattern was limited to consecutively playing the same arm multiple times. For instance, in~\cite{pike2018bandits}, the authors exploit the repetition to tackle delay and anonymous aggregation of rewards. 

Using the algorithmic template of UCBR can itself let us meet the impairment constraint (playing $d_{t,j}$ number of times in the past $N$ rounds). But since each arm is played consecutively, there is little exploration and exploitation within a phase. For instance, say there are $\armsm$ arms in phase $m$. Then each arm is played an equal number of times successively in a block of rounds, one after the other. Instead, if we reassign the number of times each arm is played so that we spend more rounds on those arms that are more promising, then our regret would be lower. We achieve this in our algorithm by playing arms in a round-robin fashion, with exploration and exploitation similar to the \aae{} algorithm~\cite{even2006action}, while ensuring that the impairment constraint is not violated. This enables us to eliminate arms during a phase, in addition to elimination between phases. Our proposed algorithm, \bis{} is described in Algorithm~\ref{alg:support_known}.

We execute the round-robin strategy described above in \bis{} by first constructing multiple subsets of arms in each phase, called \emph{buckets}, and then running \aae{} on each bucket separately. Thus, at the beginning of each phase, the arms that have not been eliminated yet, are divided up into buckets. Each bucket of arms runs a round-robin arm elimination strategy (\aae, see Algorithm~\ref{alg:aae}) that incentivizes exploration and exploitation, with the net result being that: (a) arms which are more promising in each bucket are played more, (b) the impairment constraints are satisfied due to bucket sizing, and (c) the overall regret is reduced. In our performance guarantee (Theorem~\ref{thm:support_known}), this effect of bucketing is a leading order effect. Bucketing allows our algorithm to work on a strictly larger class of impaired bandit instances when compared to a direct phase based strategy, because if the impairment constraint is tighter than the phase length (which is typically an algorithm parameter), then off-the-shelf phase base strategies will incur linear regret. Finally, note that not only is bucketing helpful in meeting an impairment constraint, it is also useful in general distributed online learning settings (with bandit feedback), where nodes can only play subsets of arms for extended periods of time and communicate with each other to identify the best-arm/minimize regret.

Reward accrual depends on how the buckets are sized. We set the size conservatively to have at most $N/\dmax$ number of arms. The arms in buckets can get eliminated before the end of the phase in addition to the possibility of getting eliminated between phases. Within the phase, an arm will get eliminated from its bucket if its mean reward is different from the estimated best arm in that bucket. Whereas, an arm will get eliminated between phases if its mean reward is different from the overall estimated best arm. As the number of phases increase, the estimates of mean rewards of the surviving arms improve. The algorithm uses a monotonically decreasing sequence of thresholds to detect sub-optimal arms, which is a standard design in phase-based methods. Controlling the thresholds, the phase lengths and the number of buckets leads to a sub-linear regret for our algorithm in the impaired bandit setting (Theorem~\ref{thm:support_known}). While the algorithm takes the time horizon $T$ as input, this can be relaxed using the \emph{doubling trick}, which is well known in the literature.\\

\noindent\textbf{Notation}: Let $m$ index phases. Let $T_j(m)$ refer to the collection of times when the $j^{th}$ arm is played up to phase $m$ and $n_{t,j}$ denote the number of times the arm $j$ was chosen till time index $t$. As discussed in earlier sections, $J_t$ refers to the arm pulled at time $t$ and the reward sequence till time index $t$ is given by $\{X_{k,J_k}\}_{k=0}^t$. The sequence of parameters $\{n_m| m=0,1,2,...\}$ determine the number of consecutive rounds each active arm is played in phase $m$, where active arms belong to the set $\armsm$. The estimated mean reward for arm $j$ at the end of phase $m$ is denoted by $\estmeanmj$. We also denote the collection of buckets in phase $m$ by $\bucketsm$. In the pseudo-code description of both the Algorithms~\ref{alg:support_known} and~\ref{alg:aae}, we assume that the flow exits whenever time index crosses the time horizon.\\

\begin{algorithm}
\SetAlgoLined
\textbf{Input:} A set of arms $\arms$, time horizon $T$, and parameters $\{n_m| m=0,1,2,...\}$.\\
\textbf{Initialization:} phase index $m=1$, $\armsm=\arms, \tilde{\Delta}_1 = 1$, $T_j(m-1) = \phi\;\; \forall j \in \arms$, and time index $t= 1$.\\
\For{ $t \leq T$}{

\emph{Within Phase Elimination}:\\
Set $\armsmp = \phi$\\
Partition $\armsm$ into a collection of buckets $\bucketsm$ of size at most $|\armsm|/(N/\dmax)$.\\
\For{each bucket $\bucketl$ in $\bucketsm$ }{
Call \aae{} (Algorithm~\ref{alg:aae}) with inputs: set of arms $\bucketl$, internal time horizon $\frac{N}{d}(n_m - n_{m-1})$,$t$ (time offset), $T_j(m-1)\;\; \forall j \in \bucketl$, and overall time horizon $T$.\\
Get outputs: $\bucketlnew \subseteq \bucketl$, $(\{X_{t,J_t}\}_{t \in T_j(m)\setminus T_j(m-1)},T_j(m)),\;\; \forall j \in \bucketlnew$, and new time index $ t^`$.\\
$t \leftarrow t^`$.\\
$\armsmp \leftarrow \armsmp \cup \bucketlnew$.
}
\emph{End of Phase Elimination:}\\
\For{each active arm $j$ in $\armsm$}{
$\quad \estmeanmj = \frac{1}{n_m}\sum_{t\in T_j(m)} X_{t,J_t}$.
}
\For{each arm $j$ in $\armsmp$}{ 
Eliminate it from $\armsmp$ if: \\
\quad $ \estmeanmj +\frac{\tilde{\Delta}_m}{2} < \max_{j' \in \armsmp } \overline{X}_{m,j'} - \frac{\tilde{\Delta}_m}{2}$.\\
}
\emph{Update the confidence bound:}\\
Set $\tilde{\Delta}_{m+1} = \frac{\dtm}{2}$.\\
Increment phase index $m$ by $1$.\\
}
\caption{\bis}\label{alg:support_known}
\end{algorithm}


\noindent\textbf{Gains from Bucketing}: The choice of the size of the bucket is conservative and ensures that all arms in a bucket satisfy the impairment constraint by construction and at most O($\dmax$) rewards are lost per arm. The parameter sequence $\{n_m| m=0,1,2,...\}$ is chosen such that the comparison step between phases, involving progressively shrinking internal thresholds ($\{\dtm\}$), eliminates suitable suboptimal arms with high probability. In particular, those arms still present in phase $m$, whose regret gaps satisfy $\Delta_i/2 > \dtm $, should ideally get eliminated.

As mentioned before, running \aae{} (Algorithm~\ref{alg:aae}) in each phase incentivizes exploration and exploitation with the possibility of eliminating arms that are found suboptimal within a bucket during the phase. \aae{} was proposed under the name \emph{Successive elimination} in~\cite{even2002pac} as a (PAC) best-arm identification algorithm and has been further developed and applied in various subsequent works ~\citep{even2006action,audibert2010best,gupta2011successive,lykouris2018stochastic}. In this algorithm, arms are played in a round robin fashion, and in each round their estimated mean reward plus a confidence bound is compared against the estimated best. A negative outcome of comparison results in the arm getting eliminated.

\begin{algorithm}
\SetAlgoLined
\textbf{Required Inputs:} A set of arms $\arms$, internal time horizon $H$.\\
\textbf{Optional Inputs:} Time offset $t$, previous counts $T_j\;\; \forall j \in \arms$, and overall time horizon $T$.\\
\textbf{Outputs:} Set of arms $\arms'$, accrued rewards $(\{X_{u,J_u}\}_{u \in T_j^{'}\setminus T_j},T_j^{'})\;\; \forall j \in \arms'$, and new time index $(t+s)$\\
\textbf{Initialization:} If optional inputs are not provided, then $s = 0, T_j=\phi \;\forall j \in \arms$, and overall time horizon $T = H$.\\
\quad Local round index $s = t+1$, $\arms' = \arms$, local counts $L_j = 0 \;\forall j \in \arms$. $T_j^{'}=T_j \;\forall j \in \arms$.\\
\While{$s\leq ( T + t)$ }{
\For{each active arm $j$ in $\arms'$}{
Play arm $J_s = j$.\\
$\overline{X}_{s,j} \leftarrow  \overline{X}_{s-1,j}$.\\
\If{ $\left(\sum_{k= \max(s-N,0)} ^s \indicator[ J_k = j]\right) \geq d_{s,j}$}{
$\quad L_j \leftarrow L_j +1$.\\
$\quad \overline{X}_{s,j} \leftarrow \frac{(L_j-1)\overline{X}_{s-1,j} + X_{s,j}}{L_j}$.\\
$\quad T_j^{'} \leftarrow T_j^{'}\cup\{t+s\}$, and also record the accrued reward $ X_{s,j}$.
}
Eliminate arm $j$ from $\arms'$ if $\overline{X}_{s,j} + \sqrt{\frac{\log T}{|T_j^{'}|}} < \max_{j' \in \arms' } \overline{X}_{s,j'} - \sqrt{\frac{\log T}{|T_{j'}^{'}|}}$.\\
$s=s+1$.\\
}
}
\caption{ Successive Elimination (with minor modifications): \aae}\label{alg:aae}
\end{algorithm}

The value of $\epsilon$, which is a property of the input instance,  governs the regret gap when \aae{} is run inside each bucket. When the bucket has largely dissimilar arms, the value of $\epsilon$ is small and \bis{} is able to explore/exploit via the buckets and eliminate suboptimal arms relatively quickly. On the other hand, when all arms have the same mean reward in a bucket, the \aae{} strategy is equivalent in performance to that of consecutively playing the same arms. In each bucket, potentially O($\dmax$) number of rewards are not accrued due to impairment (the same happens when arms are played consecutively as well). We show experimentally, that additional exploration induced by bucketing leads to a non-negligible effect in decreasing cumulative regret (see Section~\ref{sec:experiments}). We also note in passing that we did not impose any structure other than a size constraint while partitioning $\armsm$ into buckets in \bis. There could be different motivations (see Section~\ref{sec:introduction} for several motivating example) that may lead to preferred partitioning schemes, and our algorithm is general enough to be useful in those settings as well.

\begin{lemma}\label{lemma:aae}
In Algorithm~\ref{alg:aae}, the expected number of times a suboptimal arm pulled is bounded as: $\mathbb{E}[ | T_j' | ] \leq 1 + \frac{16\log(T)}{\Delta_j^2}+\frac{2}{T}$ for $s\leq T$, where $T$ is the overall time horizon, $|T_j'|$ maintains the current count of the number of times that arm $j$ was chosen and $\Delta_j$ is the regret gap.
\end{lemma}
\begin{proof}
Let $CB_{s,j} = \sqrt{\frac{\log(T)}{|T_j'|}}$ and $\mathcal{E}=\indicator{\{ \vert \overline{X}_{t,i}  -\mu_i \vert > CB_{s,j}\}}$. Firstly, we see that $\mathbb{P}(\mathcal{E})$ is upper bounded by Hoeffding's inequality by $\frac{2}{T^2}$. Now, for $T_j'$ satisfying $CB_{s,j} < \Delta_j/4$, the arm elimination condition of~\ref{alg:aae} is satisfied. This is because:
\begin{equation*}
\overline{X}_{s,j} + CB_{s,j} \leq \mu_j + 2CB_{s,j} < \mu_j + \Delta_j - 2CB_{s,j} = \mu_* - 2CB_{s,j} \leq \overline{X}_{s,*} - CB_{s,j}.
\end{equation*}
With $\Delta_j \geq 4\sqrt{\frac{\log(T)}{|T_j'|}} \rightarrow |T_j'| = |T_*'| > 1+ 16\frac{\log(T)}{\Delta_j^2}=l$, we can calculate as:
\begin{equation*}
\mathbb{E}[|T_j'|] = \sum_{s=0}^T \left( \mathbb{P}(j=J_s)\indicator\{|T_j'| \leq l\} + \mathbb{P}(j=J_s)\indicator\{|T_j'| > l\}\right),
\end{equation*}
which upon simplification gives:
\begin{equation*}
\mathbb{E}[|T_i'|] \leq 1 + \frac{16\log(T)}{\Delta_i^2} + \frac{2}{T}.
\end{equation*}
\end{proof}
\subsection{Analysis}

The choice of algorithm parameters $\{n_m| m=0,1,2,...\}$ is as follows: set $n_0 =0$, and $n_m$ as given by (\ref{eq:nm}). The following regret guarantee holds for \bis.
\begin{theorem}\label{thm:support_known}
The expected (pseudo-)regret of \bis{} (Algorithm~\ref{alg:support_known}) is upper bounded as shown below:
\begin{align}
R_T \leq& \underset{i\in\mathit{K}':\;\Delta_i>\lambda}{\sum} \frac{4\Delta_i}{T}&\nonumber\\ 
&+ \underset{i\in\mathit{K}':\;\Delta_i>\lambda}{\sum} \left( \Delta_i + \frac{16\log(T)}{\Delta_i}\min\left\{\frac{1}{(1-\epsilon)^2},4\right\}+2\Delta_i\log\left(\frac{4}{\Delta_i}\right)d_{\max} \right)&\nonumber\\ 
&+  \underset{i\in\mathit{K}''}{\sum}\frac{16}{T}+ \underset{i \in \mathit{K}'': \Delta_i < \lambda}{\max} \Delta_i T,&
\end{align}
where $\mathit{K}'' = \{ i\in \mathit{K} \vert \Delta_i >0 \}$ and $\mathit{K}' = \{ i\in \mathit{K} \vert \Delta_i > \lambda \}$.
\end{theorem}

In the above expression, the leading order terms have $\log T$ dependence. From the expression, we can also see that the upper bound gets tighter as $\epsilon$ decreases, and this is due to bucketing. When the arms are very similar to each other, then the likelihood of arms getting eliminated within a phase decreases, and we recover the performance guarantee of running only a phase-based strategy (more details on this variant are provided at the end of this section). The impairment effect leads to an additive factor that is proportional to $\dmax$. We can also get the following instance independent bound using a standard reduction. This matches the performance of several algorithms including \ucb, \ucbr{} and \aae. The bound is a factor $\sqrt{\log T}$ off from the best possible, which is achieved by the MOSS algorithm~\citep{audibert2009minimax}.

\begin{corollary}\label{coroll:support_known}
For all $T > K$, the expected (pseudo-)regret of \bis{} is upper bounded as:
\begin{align*}
    R_T \leq \mathrm{O}\left(\sqrt{KT\log T} +\sqrt{\frac{K^3\log^3(T)}{T}}\dmax\right).
\end{align*}
\end{corollary}
To get this result, we start with the regret upper bound given by the Theorem~\ref{thm:support_known}, and optimize the parameter $\lambda$ and use $\lambda = \sqrt{\frac{K\log(T)}{T}}$. We additionally note that $\log(1/\tilde{\Delta}_m) \leq \mathrm{O}(\log T )$ due to halving of the confidence estimate in each phase. The second term is the penalty due to impairment, and varies linearly with the difficulty level of the problem, i.e., $\mathrm{O}(d_{\max})$. Because $T$ is in the denominator for the term involving $\dmax$, the effect of impairment decreases quickly as the horizon length increases.

\textbf{Proof of Theorem~\ref{thm:support_known}:} The following proof bounds regret by considering four mutually exclusive and exhaustive scenarios. In a phase, let $B_{*}$ denote the bucket which would contain the optimal arm $*$. For each sub-optimal arm $i$, let $m_i = \min \{ m \vert \tilde{\Delta}_m < \frac{\Delta_i}{2}\}$ be the first phase where $\tilde{\Delta}_m < \frac{\Delta_i}{2}$. In terms of our algorithm, $m_i$ should be the first phase where arm $i$ gets eliminated with high probability. By the definition of $\tilde{\Delta}_m$ and $m_i$, we have the following inequalities:
\begin{equation*}
    2^{m_i} = \frac{1}{\tilde{\Delta}_{m_i}} \leq \frac{4}{\Delta_i} \leq \frac{1}{\tilde{\Delta}_{m_i+1}} = 2^{m_i+1}.
\end{equation*}

We upper bound the regret by bounding the expected number of times a sub-optimal arm, $i$ is pulled. Cases (a) and (c) bound the probabilities when we grossly when the reward estimate deviate from true mean by large amounts. Case (b) is for the situation when the arm is pulled exactly $n_{m_i}$ times. Finally, Case (d) accounts for regret for the nearly optimal arms.

\textbf{Case (a) } \textit{Arm $i$ is not deleted in phase $m_i$ with the optimal arm $*$ in the set $\mathit{K}_{m_i}$}.\\ 
The phase $m_i$ is characterized by : $\sqrt{\frac{\log(T)}{n_{m_i}}} \leq \frac{\Delta_{m_i}}{2} \leq \frac{\Delta_i}{4}$. 
If $i\in B_{*}$, we analyze by upper bounding the probabilities at the end of the phase. If the following inequalities hold:
\begin{equation}\label{eq:5}
    \overline{\mu}_i \leq \mu_i + \sqrt{\frac{\log(T)}{n_{m_i}}},
\end{equation}
and
\begin{equation}\label{eq:6}
    \overline{\mu}_* \geq \mu_* - \sqrt{\frac{\log(T)}{n_{m_i}}},
\end{equation}
then the end of phase elimination condition of the Algorithm~\ref{alg:support_known} are satisfied, as:
\begin{equation*}
\overline{\mu} + \sqrt{\frac{\log(T)}{n_{m_i}}} \leq \mu_* + 2\sqrt{\frac{\log(T)}{n_{m_i}}}\\
< \mu_* + \Delta_i - 2\sqrt{\frac{\log(T)}{n_{m_i}}} \\ \leq \overline{\mu}_*  - \sqrt{\frac{\log(T)}{n_{m_i}}}.
\end{equation*}
On the contrary, if the arm $i$ did not get eliminated at the end of the phase $m_i$, then~(\ref{eq:5}) and~(\ref{eq:6}) cannot hold. We can then bound the inverse of the events in~(\ref{eq:5}) and~(\ref{eq:6}) by Hoeffding's bounds as:
\begin{equation}\label{eq:51}
P( \overline{\mu}_i > \mu_i + \sqrt{\frac{\log(T)}{n_{m_i}}} ) \leq \frac{1}{T^2} .
\end{equation}
Similarly:
\begin{equation}\label{eq:61}
P( \overline{\mu}_* < \mu_* + \sqrt{\frac{\log(T)}{n_{m_i}}}) \leq \frac{1}{T^2}.
\end{equation}
By Union bound, the probability of arm $i$ not being eliminated in the phase $m_i$ is upper bounded by $\frac{2}{T^2}$. 

Let $i\notin B_{*}$ and $j$ be the locally arm optimal in that bucket. $\frac{\Delta_i-\Delta_j}{2} > \tilde{\Delta}_{m_i} $, then the arm $i$ should be eliminated inside the bucket. By the definition of $\epsilon$, we know $\frac{\Delta_i-\Delta_j}{2} > \frac{\Delta_i(1-\epsilon)}{2}$. Similar to above, the probability of the arm not getting eliminated inside the bucket but at the end of the phase is upper bounded by $\frac{2}{T^2}$. This event is subsumed by the event: arm not getting eliminated at the end of the bucket. Hence, $\forall j \in \mathit{K}'$, the expected regret conditioned on the events of this case is given by: 
\begin{equation*}
    R_T \leq \underset{i\in\mathit{K}'}{\sum}\frac{2}{T^2}(T\Delta_i) \\ \leq \underset{i\in\mathit{K}'}{\sum}\frac{2\Delta_i}{T}.
\end{equation*}

\textbf{Case (b)} \textit{ Arm $i$ is eliminated at the phase $m_i$ or during the phase $m_i$ with the optimal arm $* \in \mathit{K}_{m_i}$.}\\ 
This case marks the crucial benefit in exploration-exploitation due to use of buckets. If $i\in B_{*}$, then it would be same as if arms in $B_{*}$ were been pulled by \aae{} (Algorithm~\ref{alg:aae}) since the beginning of the horizon. So, from Lemma~\ref{lemma:aae}, 
\begin{equation*}
\mathbb{E}[n_{t,i}] \leq 1 + \frac{16\log(T)}{\Delta_i^2}+\frac{2}{T}.
\end{equation*}
If $i\notin B_{*}$, then in terms of $\epsilon$ a similar bound would hold:
\begin{equation*}
\mathbb{E}[n_{t,i}] \leq 1 + \frac{16\log(T)}{\Delta_i^2(1-\epsilon)^2}+\frac{2}{T}.
\end{equation*}
On the contrary, with bucket size$=1$ and $t=$ time at the end of phase $m_i$, $n_{m_i} = n_{t,i} \leq \left(1 + \frac{64\log(T)}{\Delta_i^2} + 2\log(4/\Delta_i)d_{\max}\right)$. Using this we can combine both the bounds as: 
\begin{equation*}
\mathbb{E}[n_{t,i}] \leq 1 + \frac{16\log(T)}{\Delta_i^2}\min\left\{\frac{1}{(1-\epsilon)^2},4\right\}+2\log(4/\Delta_i)d+\frac{2}{T}.
\end{equation*}
Hence the expected regret conditioned on the events of this case is given by:
\begin{equation*}
R_T \leq \underset{i\in\mathit{K}'}{\sum} \left( \Delta_i + \frac{16\log(T)}{\Delta_i}\min\left\{\frac{1}{(1-\epsilon)^2},4\right\}+2\Delta_i\log(4/\Delta_i)d+\frac{2\Delta_i}{T}\right),
\end{equation*}
and $\epsilon < 1/2$ offers clear advantage in terms of regret bounds when using a bucketed strategy.

\textbf{Case (c)} \textit{Optimal arm $*$ deleted by some sub-optimal $i$ in the set $\mathit{K}''$}.\\
Now, we consider the case when the last of all the optimal arms( in case there more than one ), $*$ is eliminated by some sub optimal arm $i$ in $\mathit{K}'' = \{ i\in \mathit{K} \vert \Delta_i >0 \}$ in some round $m_*$ ( Overloading the definition of $m_i$, we imply that $m_*$ is any round where $*$ is eliminated). As elimination of the optimal arm can be induced by larger number of arms, the whole set $\mathit{K}''$, at the end of the phase as compared to during a phase, we only need to analyze at the end of a phase for upper bounds.  This is similar in spirit to the event in \textbf{Case (a)}, as now as well (\ref{eq:5}) and (\ref{eq:6}) cannot hold together with the elimination condition of the Algorithm~\ref{alg:support_known}. Hence the probability of this happening is again upper bounded by $\frac{2}{T^2}$ by similar arguments. 

The optimal arm, $*$ belonged to $K_{m_s}$ for all the sub-optimal arms $s$ with $m_s < m_*$. Therefore arm $i$, which causes elimination of the optimal arm $*$, should satisfy $m_i \geq m_*$. Therefore the regret is upper bounded by:
\begin{equation*}
 R_T \leq   \overset{\max_{j\in\mathit{K}'}m_j}{\underset{m_*=0}{\sum}} \underset{i\in \mathit{K}'':m_i \leq m_*}{\sum} \frac{2}{T^2}.T \underset{j\in \mathit{K}'':m_j\geq m_*}{\max}\Delta_j,
\end{equation*}
\begin{equation*}
    \leq \overset{\max_{j\in\mathit{K}'}m_j}{\underset{m_*=0}{\sum}} \underset{i\in \mathit{K}'':m_i \leq m_*}{\sum} \frac{2}{T}4\tilde{\Delta}_{m_*},
\end{equation*}
\begin{equation*}
    \leq \underset{i\in \mathit{K}''}{\sum} \underset{m_*\geq 0}{\sum} \frac{8}{T} 2^{-m_*}.
\end{equation*}
\begin{equation*}
  R_T  \leq \underset{i\in \mathit{K}''}{\sum}  \frac{16}{T} .
\end{equation*}

\textbf{Case (d)} \textit{Arm $i\,\in\,\mathit{K}''$ and $\notin \,\mathit{K}'$}.\\ 
Here, we account for the difference in the sets $\mathit{K}''$ and $\mathit{K}'$. Following gives the upper bound on the regret conditioned on this case:
\begin{equation*}
  R_T  \leq \underset{i \in \mathit{K}'': \Delta_i < \lambda}{\max} \Delta_i T .
\end{equation*}
As all the four cases are mutually exclusive and exhaustive, we thus, get the desired regret upper bound.

\subsection{Learning without Bucketing}

When the bucket size is set to $1$, then \bis{} (Algorithm~\ref{alg:support_known}) is equivalent to \ucbr~\citep{auer2010ucb} (which is also the approach we take in Section~\ref{sec:expectation_known}, see Algorithm~\ref{alg:expectation_known}). In this case, the choice of $n_{m_i}$ is given by (\ref{eq:nm}).
\begin{lemma}\label{x1}
The expected (pseudo-)regret of Algorithm~\ref{alg:expectation_known} when $\dmax$ is known is upper bounded as:
\begin{align}
    R_T \leq& \underset{i\in\mathit{K}:\Delta_i > \lambda}{\sum}\left(\Delta_i + \frac{64\log(T)}{\Delta_i} + 2\Delta_i\log\left(\frac{4}{\Delta_i}\right)\dmax \right)&\nonumber\\
    &+\underset{i\in \mathit{K}:\Delta_i > \lambda}{\sum} \frac{2\Delta_i}{T} + \underset{i\in \mathit{K}}{\sum} \frac{16}{T} + \underset{i \in \mathit{K}: \Delta_i < \lambda}{\max} \Delta_i T,&
\end{align}
where $\mathit{K}'' = \{ i\in \mathit{K} \vert \Delta_i >0 \}$ and $\mathit{K}' = \{ i\in \mathit{K} \vert \Delta_i > \lambda \}$.
\end{lemma}
\begin{proof}
The proof is along the lines of the proof for \ucbr{} in~\cite{auer2010ucb}. In particular, we redefine the algorithm parameter $n_m$ as below to account for impairment:
\begin{equation}\label{eq:nm}
    n_m = \left\lceil \frac{4\log(T)}{\tilde{\Delta}^2_m} \right\rceil +m\dmax.
\end{equation}
The extra $m\dmax$ term allows the algorithm to play each active arm an extra $\dmax$ times in each phase, overcoming the impairment effect. As a consequence, in any phase $m$, $\estmeanmj - \mu_j \leq \tilde{\Delta}_m/2\; \forall j\in \armsm$ with high probability.
\end{proof}

When compared to the regret bound of \bis{} in Theorem~\ref{thm:support_known}, we can see that the above bound will be worse for input instances whose corresponding $\epsilon$ is small. And this non-trivial difference is due to the use of buckets, each running \aae.

\section{Known Expected Impairment}\label{sec:expectation_known}
Just like in the previous section and also similar to the algorithm designs of~\cite{auer2010ucb,karnin2013almost,valko2014spectral} and~\cite{pike2018bandits}, our algorithm also works in phases. In each phase, every arm in the set of \emph{active arms} is played consecutively (unlike Section~\ref{sec:support_known} where there are played in a round-robin fashion inside buckets, here the bucket size can be considered as $1$). Again we argue that repetition ensures that the stochastic impairment is limited. As in  these previous works, the rewards collected for each arm are used to update the estimate of the mean rewards. Between two consecutive phases, arms with low estimated mean rewards are eliminated (unlike Section~\ref{sec:support_known} where they can be eliminated in any round). The threshold chosen to eliminate the arms is updated sequentially. 

As more number of phases are executed, the estimated mean rewards for the active arms get closer to their true means and the threshold is suitably adjusted. Because of repeat plays of the same arm for a carefully chosen number of rounds, the effect of reward impairment is mitigated. Similar to previous works, the algorithm has access to the horizon length $T$. This information can be relaxed by adopting the \emph{doubling trick} as mentioned in Section~\ref{sec:support_known}.
Finally, note that in~\cite{pike2018bandits}, the algorithm has a bridge period to address the issue of aggregation of delayed rewards. In~\cite{auer2010ucb}, the algorithm chooses a different length for each phase, compared to our choice. They also have a different function of $\dtm$ in the arm elimination step.\\

\noindent\textbf{Notation}: Let $m$ index phases. Let $T_j(m)$ refer to the collection of times when the $j^{th}$ arm is played up to phase $m$. As in previous section, $J_t$, is the arm pulled at time $t$, $X_t$ is the reward accrued at time $t$ and $R_{t,j}$ is the rewards observed on playing arm $j$ at time $t$. The sequence of parameters $\{n_m| m=0,1,2,...\}$ determine the number of consecutive rounds each active arm is played in phase $m$, where active arms belong to the set $\armsm$. The estimated mean reward for arm $j$ at the end of phase $m$ is denoted by $\estmeanmj$.\\

\begin{algorithm}
\SetAlgoLined
\textbf{Input:} A set of arms $\arms$, time horizon $T$, and parameters $\{n_m| m=0,1,2,...\}$.\\
\textbf{Initialization:} phase index $m=1$, $\armsm = \arms$, $\tilde{\Delta}_{1}$ = 1, $T_j(m) = \phi\;\; \forall j \in \arms$, and time index $t= 1$.\\
\While{$t\leq T$ }{
\emph{Play arms:}\\
\For{ each active arm $j$ in $\armsm$}{
Set $T_j(m) = T_j(m-1)$ if $m > 1$.\\
Play $j$ for $n_m-n_{m-1}$ consecutive rounds. In each round $t$:\\
\If{ $\left(\sum_{k= \max(t-N,0)}^t \indicator[ J_k = j]\right) \geq d_{t,j}$}{
    \quad\quad Observe reward $X_{t,j}$ and add $t$ to $T_j(m)$.\\
}
}
\emph{Eliminate Suboptimal Arms:}\\
\For{each active arm $j$ in $\armsm$}{
$\quad \estmeanmj = \frac{1}{n_m}\sum_{t\in T_j(m)} X_{t,j}$.
}
Construct $\armsmp$ by eliminating arms $j$ in $\armsm$ for which\\
\quad \quad $\estmeanmj + \tilde{\Delta}_m/2  < \max_{j' \in \armsm} \overline{X}_{m,j'} - \tilde{\Delta}_m/2$.\\
\emph{Update the confidence bound:}\\
Set $\tilde{\Delta}_{m+1} = \frac{\tilde{\Delta}_m}{2}$.\\
Increment phase index $m$ by $1$.
}
\caption{\bie}\label{alg:expectation_known}
\end{algorithm}

\subsection{Analysis}

Unlike the result in Section~\ref{sec:support_known}, here we perform a more careful analysis to choose the values of the sequence $\{n_m|m=0,1,2...\}$. This yields the following regret upper bound, where the dependence on the expected impairment value $\expd$ is sublinear. In contrast, the regret upper bound for \bis{} depends linearly in $\dmax$. 
\begin{theorem}\label{thm:expectation_known} The expected (pseudo-)regret of \bie{} (Algorithm~\ref{alg:expectation_known}) is bounded as:
\begin{align}
 R_T \leq  & \underset{i\in\mathit{K'}}{\sum}\left(\Delta_i + \frac{64 \log(T)}{\Delta_i} + \frac{64\log(T)}{3} + 32\sqrt{\log\left(\frac{4}{\Delta_i}\right)\mathbb{E}[d]\log(T)} \right)&\nonumber\\
 &+\underset{i\in \mathit{K}':\Delta_i > \lambda}{\sum} \frac{2\Delta_i}{T} + \underset{i\in \mathit{K}'}{\sum} \frac{32}{T} + \underset{i \in \mathit{K}'': \Delta_i < \lambda}{\max} \Delta_i T,&
\end{align}
where $\mathit{K}'' = \{ i\in \mathit{K} \vert \Delta_i >0 \}$ and $\mathit{K}' = \{ i\in \mathit{K} \vert \Delta_i > \lambda \}$.
\end{theorem}

Similar to the Corollary~\ref{coroll:support_known}, an instance independent bound for this setting is shown below.
\begin{corollary}\label{coroll:expectation_known}
For all $T \geq K$ and choosing $\lambda=\sqrt{\frac{K\log(T)}{T}}$ the expected (pseudo-)regret of \bie{} is upper bounded as:
\begin{align*}
    R_T \leq \mathrm{O}\left(\sqrt{KT\log T} + K\sqrt{\log^2 T \expd}\right).
\end{align*}
\end{corollary}

The key difference with Corollary~\ref{coroll:support_known} is that in above there is milder dependence on the impairment parameter ($\sqrt{\expd}$ vs $\dmax$) but with a worse time dependence in the second term. However, here Algorithm~\ref{alg:expectation_known} is doing a significantly harder job as only $\expd$ is available to it.
The proof of the result in Theorem~\ref{thm:expectation_known} is along the lines of an analysis for a phase based algorithm called \odaaf{} that was designed to address delays in reward accrual in~\cite{pike2018bandits}. It poses very different challenges, e.g. rewards due to multiple arms may get accrued collectively, at a given time instance in the delay model (see Section~\ref{sec:introduction} for a detailed comparison of these two models). In impairment setting learning occurs by repeating actions appropriate number of times. Even though the philosophy of (\ref{eq:stoc1}) is very similar to (13) in~\cite{pike2018bandits}, we obtain different order of regret bounds.

There are two key aspect to the proof: (a) identifying an appropriate $n_m$ that can work with stochastic impairment effects, (b) showing that this $n_m$ swiftly eliminates the sub-optimal arms in the execution of \bie{} (Algorithm~\ref{alg:expectation_known}). We first provide several key lemmas that will be useful to prove the main claim, which we defer to until after Lemma~\ref{term1}.

We first start with two concentration inequalities concerning certain classes of stochastic processes below:
\begin{lemma}\label{freedman}
\textbf{Generalized Bernstein inequality for Martingales} (Theorem 1.6 in~\cite{freedman1975tail}, Theorem 10 in~\cite{pike2018bandits}) Let $\{Y_k\}_{k=0}^{\infty}$ be a real valued martingale with respect to the filtration, $\{\mathcal{F}_k\}_{k=0}^{\infty}$ with increments $\{Z_k\}_{k=1}^{\infty}$, implying $\mathbb{E}[Z_k \vert \mathcal{F}_{k-1}] = 0$ and $Z_k = Y_k - Y_{k-1}$ for $k=1,2,3...$. Given the martingale difference sequence is uniformly upper bounded as, $Z_k \leq b$ for $k=1,2,3...$. Define the predictable variation process $W_k = \sum_{j=1}^k \mathbb{E}[Z_j^2 \vert \mathcal{F}_{j-1}]$ for $k=1,2...$. Then for all $\alpha \geq 0$, $\sigma^2 \geq 0$, the following probability is bounded:
\begin{equation*}
    \mathbb{P}\left( \exists k \; :\; Y_k \geq \alpha \; \textrm{ and } W_k \leq \sigma^2 \right) \leq \exp\left(-\frac{\alpha^2/2}{\sigma^2+b\alpha/3}\right).
\end{equation*}
\end{lemma}

\begin{lemma}\label{doob}
\textbf{Concentration of bounded random variables} (Lemma A.1 in~\cite{szita2011agnostic}, Lemma 11 in~\cite{pike2018bandits}) Fix the positive integers $m$, $n$ and let $a,c \in \mathbb{R}$. Let $\mathcal{F}= \{\mathcal{F}_t\}_{t=0}^{\infty}$ be a filtration, $(\rho_t)_{t=1,2,3...n}$ be $\{0,1\}$ valued and $\mathcal{F}_{t-1}$-measurable random variables and $(Z_t)_{t=1,2,3...n}$ be $\mathcal{F}_t$-measurable $\mathbb{R}$ valued random variables, satisfying $\mathbb{E}[Z_t\vert\mathcal{F}_{t-1}]=0$, $Z_t\in[a,a+c]$ and $\sum_{s=1}^n\rho_s \leq m$ with probability one. Then for any $\eta >0$:
\begin{equation*}
    \mathbb{P} \left( \sum_{t=1}^n \rho_t Z_t \geq \eta \right) \leq \exp\left(-\frac{2\eta^2}{c^2m} \right ). 
\end{equation*}
\end{lemma}

Now, we setup some additional notation. For any active arm $j$ and phase $m$ , let $S_{m,j}$ denote the time in this phase when the algorithm starts playing this arm. Similarly let $U_{m,j}$ denote the time in this phase when the algorithm stops playing this arm. All the phases are distinct and non-overlapping. Except for an initial few plays, reward accrual is instantaneous by construction. We define a filtration $\{\mathcal{G}_s\}_{s=0}^{\infty}$ by setting $\{\mathcal{G}_0\} = \{\Omega, \phi \}$ and defining $\{\mathcal{G}_t\}$ to be the $\sigma$-algebra over $(X_{1},X_2....X_t,J_1,J_2....J_t,d_{1,J_1},d_{2,J_2}....d_{t,J_t},R_{1,J_1},R_{2,J_2}...R_{t,J_t})$.

In the next lemma, we give a constructive proof for the choice of $n_m$ such that the estimated mean reward for an arm $j$ is close to its true mean within a predefined error range with high probability.
\begin{lemma}\label{stocnm}
There exists a positive $n_m$ for which the estimates $\overline{X}_{m,j}$, calculated by the Algorithm~\ref{alg:expectation_known} for the active arm $j$ ( $j\in \mathit{K}_m$ ) and phase $m$, satisfy the following with probability $\geq(1-\frac{2}{T^2})$:
\begin{equation*}
    \overline{X}_{m,j} - \mu_j \leq \tilde{\Delta}_m/2.
\end{equation*}
\end{lemma}

\begin{proof}
Using the above notation, it follows that for each arm $j$:
\begin{equation}\label{eq:stoc1}
    \sum_{i=1}^m\sum_{t=S_{i,j}}^{U_{i,j}} ( X_t - \mu_j ) \leq \sum_{i=1}^m\sum_{t=S_{i,j}}^{U_{i,j}} ( R_{t,J_t} - \mu_j ) - \sum_{i=1}^m\sum_{t=S_{i,j}}^{U_{i,j}} R_{t,J_t}\indicator\{ t \leq S_{i,j} + d_{t,J_t}\} .
\end{equation}
Define $A_{i,t} := R_{t,J_t}\indicator\{ t \leq S_{i,j} + d_{t,J_t}\}$ and also $M_t := \sum_{i=m}^m A_{i,t}\indicator\{ S_{i,j}\leq t \leq U_{i,j} \}$. We rewrite~(\ref{eq:stoc1}) in terms of $M_t$ as:
\begin{equation*}
    \sum_{i=1}^m\sum_{t=S_{i,j}}^{U_{i,j}} ( X_t - \mu_j ) \leq \sum_{i=1}^m\sum_{t=S_{i,j}}^{U_{i,j}} ( R_{t,J_t} - \mu_j ) -\sum_{t=1}^{U_{m,j}} M_t,
\end{equation*}
\begin{equation}\label{eq:71}
    \sum_{i=1}^m\sum_{t=S_{i,j}}^{U_{i,j}} ( X_t - \mu_j ) \leq \sum_{i=1}^m\sum_{t=S_{i,j}}^{U_{i,j}} ( R_{t,J_t} - \mu_j ) +\sum_{t=1}^{U_{m,j}} (\mathbb{E}[M_t\vert G_{t-1}] - M_t )- \sum_{t=1}^{U_{m,j}}\mathbb{E}[M_t\vert G_{t-1}].
\end{equation}
We bound each term individually in~(\ref{eq:71})
From lemmas~\ref{term1} and~\ref{term2} accompanied by a trivial non-negative upper bound for the term 3 in~(\ref{eq:71}), we can write with probability $1-\frac{2}{T^2}$ ( from union bound )
\begin{equation*}
    \sum_{i=1}^m\sum_{t=S_{i,j}}^{U_{i,j}} ( X_t - \mu_j ) \leq \sqrt{n_m \log(T)} + \frac{2}{3}\log(T) + \sqrt{\frac{4\log^2(T)}{9}+4m\mathbb{E}[d]\log(T)}.
\end{equation*}
For each active arm $j \in \mathit{K}_m$,
\begin{equation*}
    \frac{1}{n_m}\sum_{t \in T_j(m)} ( X_t - \mu_j ) \leq \sqrt{\frac{\log(T)}{n_m}} + \frac{2\log(T)}{3n_m} + \frac{1}{n_m}\sqrt{\frac{4\log^2(T)}{9}+4m\mathbb{E}[d]\log(T)} = w_m.
\end{equation*}
Algorithm~\ref{alg:expectation_known} requires $w_m \leq \tilde{\Delta}_m/2$ so that the arm elimination condition holds good. This helps to determine appropriate $n_m$. Let $z= \sqrt{\frac{4\log^2(T)}{9}+4m\mathbb{E}[d]\log(T)}$, then, we need to solve the following quadratic inequality:
\begin{equation*}
    \frac{\tilde{\Delta}_m}{2}n_m - \sqrt{n_m \log(T)} - \frac{2\log(T)}{3} - z \geq 0.
\end{equation*}
The smallest $n_m$ which satisfies the above is given by:
\begin{equation*}
    n_m = \left\lceil\frac{1}{\tilde{\Delta}^2_m}\left( \sqrt{\log(T)} + \sqrt{\log(T) + \frac{4\tilde{\Delta}_m\log(T)}{3} + 2\tilde{\Delta}_m z } \right )^2\right\rceil.
\end{equation*}
Using $(a+b)^2 \leq 2(a^2+b^2)$ and $x=\lceil y \rceil \rightarrow x \leq y+1$:
\begin{equation*}
    n_m \leq 1 + \frac{1}{\tilde{\Delta}^2_m}\left( 4 \log(T) + \frac{8\tilde{\Delta}_m\log(T)}{3} + 4\tilde{\Delta}_mz\right ).
\end{equation*}
We can now substitute $z$ and use inequality $\sqrt{a^2+b^2}\leq (a+b)$, therefore:
\begin{equation*}
n_m \leq 1 + \frac{1}{\tilde{\Delta}^2_m}\left( 4 \log(T) + \frac{16\tilde{\Delta}_m\log(T)}{3} + 8\tilde{\Delta}_m\sqrt{m\mathbb{E}[d]\log(T)}\right ),
\end{equation*}
\begin{equation}\label{eq:stocnmval}
n_m \leq 1 + \frac{4 \log(T)}{\tilde{\Delta}^2_m} + \frac{16\log(T)}{3\tilde{\Delta}_m} + \frac{8\sqrt{m\mathbb{E}[d]\log(T)}}{\tilde{\Delta}_m}.
\end{equation}
This completes the proof.
\end{proof}
\begin{lemma}\label{martingaleproof}
$Y_s = \sum_{t=1}^s \mathbb{E}[M_t\vert G_{t-1}] - M_t)$ for all $s\geq 1$ with $Y_0=0$ is a martingale with respect to the filtration $\{G_s\}^{\infty}_{s=0}$ with increments $Z_s = Y_s-Y_{s-1} = \mathbb{E}[M_s\vert G_{s-1}] - M_s )$, satisfying $\mathbb{E}[Z_s\vert G_{s-1}] = 0$, $Z_s \leq 1$ for all $s\geq 1$.
\end{lemma}
\begin{proof}
To show $\{Y_s\}_{s=0}^{\infty}$ is a martingale defined on filtration $\{\mathcal{G}_s\}_{s=0}^{\infty}$, we need to show $Y_s$ is $\{\mathcal{G}_s\}$-measurable for all $s\geq 1$ and $\mathbb{E}[Y_s\vert G_{s-1}] =Y_{s-1}$

By the definition of $\sigma$-algebra, $\{\mathcal{G}_s\}_{s=0}^{\infty}$, $d_{t,J_t},R_{t,J_t}$ are all $\{\mathcal{G}_s\}$-measurable for $t\leq s$. Additionally for phases $i$ where time instance $t$ lie in phases after $i$, $\indicator\{ S_{i,j}\leq t \leq U_{i,j} \}=0$ ( measurable by $\mathcal{G}_0$ ). Hence $\{Y_s\}_{s=0}^{\infty}$ is measurable by $\{\mathcal{G}_s\}_{s=0}^{\infty}$. Now consider the conditional expectation:
\begin{equation*}
    \mathbb{E}[Y_s\vert \{\mathcal{G}_{s-1}\}] = \mathbb{E}\left[\sum_{t=1}^s (\mathbb{E}[M_t\vert G_{t-1}] - M_t) \vert \mathcal{G}_{s-1} \right ],
\end{equation*}
\begin{equation*}
    = \mathbb{E}\left[\sum_{t=1}^{s-1} (\mathbb{E}[M_t\vert G_{t-1}] - M_t) \vert \mathcal{G}_{s-1} \right ] + \mathbb{E}\left[ (\mathbb{E}[M_s\vert G_{s-1}] - M_s) \vert \mathcal{G}_{s-1} \right ],
\end{equation*}
\begin{equation*}
    = \mathbb{E}\left[\sum_{t=1}^{s-1} (\mathbb{E}[M_t\vert G_{t-1}] - M_t) \vert \mathcal{G}_{s-1} \right ]  = Y_{s-1}.
\end{equation*}
Therefore $\{Y_s\}_{s=0}^{\infty}$ is a martingale with respect to the filtration $\{\mathcal{G}_s\}_{s=0}^{\infty}$. Clearly, the increment $Z_s = Y_s - Y_{s-1} = (\mathbb{E}[M_s\vert G_{s-1}] - M_s)$ and $\mathbb{E}[Z_s\vert G_{s-1}] = \mathbb{E}[(\mathbb{E}[M_s\vert G_{s-1}] - M_s)\vert G_{s-1} ] = 0$ 

Note that for any phase $i$, $A_{i,t} \leq 1$ as reward $R_{t,J_t}$ is bounded by $1$. Also, for any time $t$ $M_t \leq 1$ as  $\indicator\{ S_{i,j}\leq t \leq U_{i,j} \}$ is $1$ for only a particular phase $i$ thus $Z_s \leq 1$ for $s \geq 1$. This completes the proof
\end{proof}

\begin{lemma}\label{term1}
Using the notations derived before:
\begin{equation*}
    \mathbb{P}\left( \sum_{i=1}^m\sum_{t=S_{i,j}}^{U_{m,j}}( R_{t,J_t} -\mu_j ) \leq \sqrt{n_m \log(T)} \right) > 1 - \frac{1}{T^2}.
\end{equation*}
\end{lemma}
\begin{proof}
We would apply lemma~\ref{doob} to prove the above. Take $n=T$, $\mathcal{F}_t$ as filtration with $\sigma$-algebra on $(X_1,....X_t,R_{1,j}...R_{t,j})_{t=1,2...T}$. $Z_t = R_{t,j}-\mu_j$ and $\rho_t = \indicator\{J_t=j,\;t\leq U_{m,j}\}$. Clearly, $\forall t \in T_j(m),\;\rho_t =1$ and by definition of $T_j(m)$, $\sum_{t=1}^T\rho_t = \vert T_j(m) \vert \leq n_m$. So $\sum_{t\in T_j(m)} (R_{t,j} -\mu_j) = \sum_{t=1}^T \rho_t(R_{t,j}-\mu_j)$ follows.

Additionally, for any $1\leq t \leq T$, $\rho_t = \indicator\{J_t=j,\;t\leq U_{m,j}\}$ is $\mathcal{F}_{t-1}$-measurable. Given all the observations $X_1,X_2....X_{t-1}$ till $(t-1)$ we know which phase does $t$ belong. This is because of the phased the phased nature of Algorithm~\ref{alg:expectation_known}, thus, $\indicator\{t\leq U_{m,j}\}$ is determined. Similarly $J_t=j$ can be determined, establishing $\rho_t$ is $\mathcal{F}_{t-1}$-measurable. $Z_t$ is $\mathcal{F}_{t}$-measurable by definition. Taking $a=-\mu_j$ and $c=1$, we apply lemma~\ref{doob} to get the above result.
\end{proof}
\begin{lemma}\label{variation}
For any $t$, $Z_t= \mathbb{E}[M_t\vert G_{s-1}] - M_t$, then
\begin{equation*}
    \sum_{t=1}^{U_{m,j}} \mathbb{E}[Z_t^2\vert G_{s-1}] \leq m\mathbb{E}[d].
\end{equation*}
\end{lemma}
\begin{proof}
\begin{equation*}
\sum_{t=1}^{U_{m,j}} \mathbb{E}[Z_t^2\vert G_{s-1}] = \sum_{t=1}^{U_{m,j}} \mathbb{V}[M_t^2\vert G_{s-1}] \leq \sum_{t=1}^{U_{m,j}} \mathbb{E}[M_t^2\vert G_{s-1}]
 = \sum_{t=1}^{U_{m,j}}\mathbb{E}\left[\left( \sum_{i=1}^m A_{i,t}\indicator \{ S_{i,j} \leq t \leq U_{i,j} \} \right)^2 \vert G_{t-1} \right].
\end{equation*}
For phase $i$, the indicator $\indicator \{ S_{i,j} \leq t \leq U_{i,j} \}=1$ for atmost once. Therefore:
\begin{equation*}
\sum_{t=1}^{U_{m,j}} \mathbb{E}[Z_t^2\vert G_{s-1}] \leq \sum_{t=1}^{U_{m,j}}\mathbb{E}\left[\sum_{i=1}^m A_{i,t}^2\indicator \{ S_{i,j} \leq t \leq U_{i,j} \} \vert G_{t-1} \right ], 
\end{equation*}
\begin{equation*}
= \sum_{i=1}^m\sum_{t=1}^{U_{m,j}}\mathbb{E}\left[\sum_{i=1}^m A_{i,t}^2\indicator \{ S_{i,j} \leq t \leq U_{i,j} \} \vert G_{t-1} \right ],     
\end{equation*}
\begin{equation*}
= \sum_{i=1}^m\sum_{t=S_{i,j}}^{U_{m,j}}\mathbb{E}\left[ A_{i,t}^2\indicator \{ S_{i,j} \leq t \leq U_{i,j} \} \vert G_{t-1} \right ].    
\end{equation*}
As $S_{i,j}$ and $U_{i,j}$ are $G_{t-1}$-measurable, if $\indicator \{ S_{i,j} \leq t \leq U_{i,j} \}=1$, therefore:
\begin{equation*}
\leq \sum_{i=1}^m\sum_{t=S_{i,j}}^{U_{i,j}}\mathbb{E}[A_{i,t}^2\vert G_{t-1}],     
\end{equation*}
\begin{equation*}
= \sum_{i=1}^m\sum_{t=S_{i,j}}^{U_{i,j}}\mathbb{E}[R_{t,J_t}^2\indicator\{ t < S_{i,j} + d_{t,J_t} \} \vert G_{t-1}],    
\end{equation*}
\begin{equation*}
\leq \sum_{i=1}^m\sum_{t=S_{i,j}}^{U_{i,j}}\mathbb{E}[\indicator\{ t < S_{i,j} + d_{t,J_t} \} \vert G_{t-1}],     
\end{equation*}
\begin{equation*}
= \sum_{i=1}^m\sum_{s=0}^{\infty}\sum_{s'=s}^{\infty}\sum_{t=s}^{s'}\mathbb{E}[\indicator\{S_{i,j}=s,\;U_{i,j}=s',\; t < s + d_{t,J_t} \} \vert G_{t-1}], 
\end{equation*}
\begin{equation*}
= \sum_{i=1}^m\sum_{s=0}^{\infty}\sum_{s'=s}^{\infty}\sum_{t=s}^{s'}\indicator\{S_{i,j}=s,\;U_{i,j}=s'\}\sum_{t=s}^{s'}\mathbb{P}( t < s + d_{t,J_t} ),   
\end{equation*}
\begin{equation*}
\leq \sum_{i=1}^m\sum_{s=0}^{\infty}\sum_{s'=s}^{\infty}\sum_{t=s}^{s'}\indicator\{S_{i,j}=s,\;U_{i,j}=s'\}\sum_{l=0}^{\infty}\mathbb{P}( l <  d ),  
\end{equation*}
\begin{equation*}
\leq \sum_{i=1}^m \mathbb{E}[d] = m\mathbb{E}[d].    
\end{equation*}
This completes the proof.
\end{proof}

\begin{lemma}\label{term2}
Using the notations derived before:
\begin{equation*}
    \mathbb{P}\left( \sum_{t=1}^{U_{m,j}}(\mathbb{E}[M_t\vert G_{t-1}] - M_t ) < \frac{2}{3}\log(T) + \sqrt{\frac{4\log^2(T)}{9}+4m\mathbb{E}[d]\log(T)} \right ) > 1-\frac{1}{T^2}.
\end{equation*}
\end{lemma}
\begin{proof}
Using Lemma~\ref{martingaleproof}, we know $Y_s = \sum_{t=1}^s \mathbb{E}[M_t\vert G_{t-1}] - M_t)$ for all $s\geq 1$ with $Y_0=0$ is a martingale with respect to the filtration $\{G_s\}^{\infty}_{s=0}$ with increments $Z_s = Y_s-Y_{s-1} = \mathbb{E}[M_s\vert G_{s-1}] - M_s )$, satisfying $\mathbb{E}[Z_s\vert G_{s-1}] = 0$, $Z_s \leq 1$ for all $s\geq 1$. Also from lemma~\ref{variation}, $\sum_{t=1}^s\mathbb{E}[Z_t^2\vert G_{t-1}] \leq m\mathbb{E}[d]$. As lemma~\ref{freedman} guarantees the existence of a $t$ for which the above inequality is satisfied with high probability. 
\end{proof}

\textbf{Proof of Theorem~\ref{thm:expectation_known}:}
\begin{proof}
Proceeding as in the proof of Theorem~\ref{thm:support_known}, we create $4$ mutually exclusive and exhaustive cases with very similar interpretation with the difference being in this setting elimination at the end of the phase is only considered.  We then bound the expected regret conditioned on the events of these cases. Recall that for each sub-optimal arm $i$, let $m_i := \min \{ m \vert \tilde{\Delta}_m < \frac{\Delta_i}{2}\}$, is the first phase where $\tilde{\Delta}_m < \frac{\Delta_i}{2}$. We bound the expected regret by following cases:\\

\textbf{Case (a)} \textit{Arm $i$ is not deleted in phase $m_i$ with the optimal arm $*$ in the set $\mathit{K}_{m_i}$}.\\
The phase $m_i$ is characterized by : $w_{m_i} \leq \frac{\Delta_{m_i}}{2} \leq \frac{\Delta_i}{4}$. Let $E:=\indicator\{\overline{\mu}_i \leq \mu_i + w_{m_i}\}$ and $R:=\indicator\{\overline{\mu}_* \geq \mu_* - w_{m_i}\}$.

If the events $E$ and $R$ hold then the phase-end elimination condition of the Algorithm~\ref{alg:expectation_known} are satisfied, as:
\begin{equation*}
\overline{\mu}_i + w_{m_i} \leq \mu_i + 2w_{m_i} < \mu_i + \Delta_i - 2w_{m_i} \leq \overline{\mu}_*  - w_{m_i}.
\end{equation*}
From the Lemma~\ref{stocnm}, $\mathbb{P}(E > 1-\frac{1}{T^2})$ and $\mathbb{P}(E > 1-\frac{1}{T^2})$ follows. In this case we are interested in conditional regret of $\{E^{\complement} \cup \R^{\complement}\}$ which by Union bound bears the upper bound of $\frac{4}{T^2}$. which is given by:
\begin{equation*}
    \sum_{i\in \mathit{K}'} \frac{4}{T^2} T\Delta_i \leq \sum_{i\in \mathit{K}'} \frac{4\Delta_i}{T}.
\end{equation*}

\textbf{Case (b)} \textit{ Arm $i$ is eliminated at the phase $m_i$ with the optimal arm $* \in \mathit{K}_{m_i}$.}\\ 
By the lemma~\ref{stocnm}, in case the sub-optimal arm $i$ is eliminated in the phase $m_i$ then the maximum number of times it is played is given by~(\ref{eq:stocnmval}). Additionally, we make use of $\Delta_i/4 \leq \tilde{\Delta}_m\leq \Delta_i/2$ and $m_i \leq \log_2\left(\frac{4}{\Delta_i}\right) < 2\log\left(\frac{4}{\Delta_i}\right)$.Therefore:
\begin{equation*}
n_{m_i} \leq   1 + \frac{64 \log(T)}{\Delta_i^2} + \frac{64\log(T)}{3\Delta_i} + \frac{32\sqrt{\log\left(\frac{4}{\Delta_i}\right)\mathbb{E}[d]\log(T)}}{\tilde{\Delta}_i}
\end{equation*},
\begin{equation*}
R_T  \leq  \sum_{i\in\mathit{K}'}\Delta_i\left( 1 + \frac{64 \log(T)}{\Delta_i^2} + \frac{64\log(T)}{3\Delta_i} + \frac{32\sqrt{\log\left(\frac{4}{\Delta_i}\right)\mathbb{E}[d]\log(T)}}{\Delta_i} \right),
\end{equation*}
\begin{equation*}
R_T  \leq  \sum_{i\in\mathit{K}'}\left( \Delta_i + \frac{64 \log(T)}{\Delta_i} + \frac{64\log(T)}{3} + 32\sqrt{\log\left(\frac{4}{\Delta_i}\right)\mathbb{E}[d]\log(T)} \right).
\end{equation*}

The following two cases exactly the same as for the proof of Theorem~\ref{thm:support_known}.

\textbf{Case (c)} \textit{Optimal arm $*$ deleted by some sub-optimal $i$ in the set $\mathit{K}''$}.\\
Recall the events $E$ and $R$ as discussed in \textbf{Case (a)}, the event in which $*$ is eliminated by some sub optimal arm $i$ in $\mathit{K}'' = \{ i\in \mathit{K} \vert \Delta_i >0 \}$, is simply $\{E^{\complement} \cup \R^{\complement}\}$ with role of $*$ and $i$ reversed. Hence the probability of this happening is again upper bounded by $\frac{4}{T^2}$ by similar arguments.Let $m_*$ imply the round where $*$ is eliminated. 

Following the same argument, as the \textbf{Case (c)} of the previous section, arm $i$, which causes elimination of the optimal arm $*$, should satisfy $m_i \geq m_*$. Therefore the regret is upper bounded by:
\begin{equation*}
  R_T \leq  \overset{\max_{j\in\mathit{K}'}m_j}{\underset{m_*=0}{\sum}} \underset{i\in \mathit{K}'':m_i \leq m_*}{\sum} \frac{4}{T^2}.T \underset{j\in \mathit{K}'':m_j\geq m_*}{\max}\Delta_j,
\end{equation*}
\begin{equation*}
    \leq \overset{\max_{j\in\mathit{K}'}m_j}{\underset{m_*=0}{\sum}} \underset{i\in \mathit{K}'':m_i \leq m_*}{\sum} \frac{4}{T}4\tilde{\Delta}_{m_*},
\end{equation*}
\begin{equation*}
    \leq \underset{i\in \mathit{K}''}{\sum} \underset{m_*\geq 0}{\sum} \frac{16}{T} 2^{-m_*},
\end{equation*}
\begin{equation*}
    \leq \underset{i\in \mathit{K}''}{\sum}  \frac{32}{T}.
\end{equation*}

\textbf{Case (d)} \textit{Arm $i\,\in\,\mathit{K}''$ and $\notin \,\mathit{K}'$}.\\ 
 Following gives the upper bound on the regret conditioned on this case:
\begin{equation*}
  R_T \leq \underset{i \in \mathit{K}'': \Delta_i < \lambda}{\max} \Delta_i T .
\end{equation*}
As all the four cases mutually exclusive and comprehensive, we thus, conclude the proof of the Theorem~\ref{thm:expectation_known}.
\end{proof}

\section{Experiments}\label{sec:experiments}

We run three experiments. In the first, we investigate the arm switching behavior of \ucb{}. In the second, we investigate the effect of bucketing when there is no impairment present. In the third, we show how \bie{} performs as a function of the impairment effect. In all these experiments, the cumulative regret curves plotted were averaged over 30 Monte Carlo runs. The bandit instances were generated randomly, unless otherwise noted.

We use a simple setup of $K = 30$ arms and set the reward distributions to be the Bernoulli with randomly chosen biases. The horizon length $T = 5000$. We then run \ucb under three different configurations. In the first, the bandit instance is run as is and there is an unique optimal arm. In the second, the number of optimal arms is increased to $3$, and in the third the number of optimal arms is increased to $7$. Figure~\ref{fig:7} shows the unnormalized counts of \emph{same arm plays} in the past $15$ plays. This was computed by checking how many times the current arm was also played in the past $15$ rounds. As expected, as the number of optimal arms increases, the counts of same arm plays decreases rapidly. This indicates that \ucb{} and other related algorithms may perform poorly in settings with impairment.

\begin{figure}
\centering
	\includegraphics[width=.5\columnwidth]{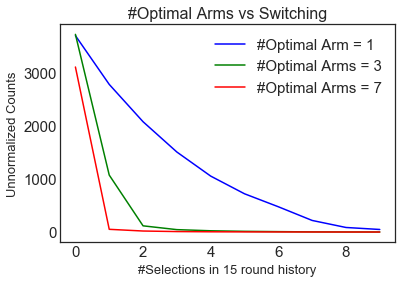}
\caption{Plot of unnormalized counts versus number of same arm plays in the past $15$ rounds by \ucb{} for three different settings.}
\label{fig:7}
\end{figure}

Next, we compared the performance of \bis{} (Algorithm \ref{alg:support_known}) against \aae{} and \ucb{} when there is no impairment. While our algorithms are expected to do worse than \ucb{} and \aae, our aim here is to investigate the role of bucket size. Our instance has $20$ arms, and we choose a time horizon of $5000$ rounds. The distributions associated with each arm follow Bernoulli with a randomly chosen bias. As shown in Figures~\ref{fig:2} and~\ref{fig:3}, the performance of \bis{} interpolates between that of \aae when the bucket size is large ($= 20$) and \ucbr when the bucket size is small ($= 3$).

\begin{figure}
\centering
\includegraphics[width=.45\columnwidth]{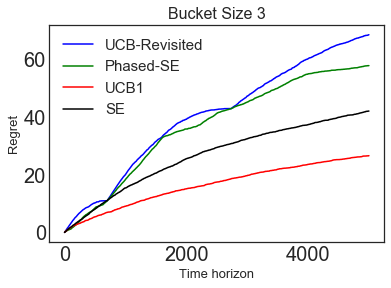}{}
\includegraphics[width=.45\columnwidth]{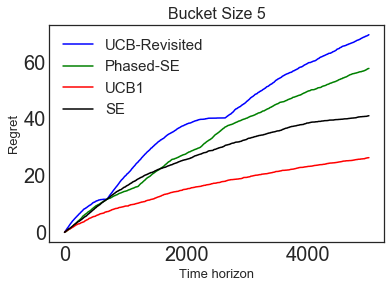}
\caption{Cumulative regret as a function of time: performance of \bis{} for small bucket sizes is closer to the performance of \ucbr.}
\label{fig:2}
\end{figure}

\begin{figure}
\centering
\includegraphics[width=.45\columnwidth]{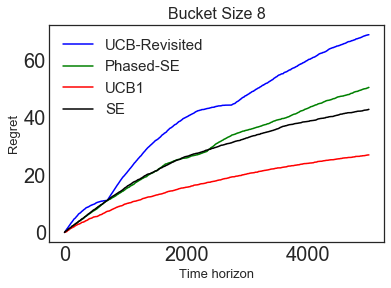}{}
\includegraphics[width=.45\columnwidth]{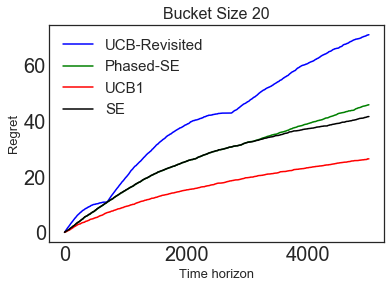}
\caption{Cumulative regret as a function of time: performance of \bis{} for large bucket sizes is closer to the performance of \aae.}
\label{fig:3}
\end{figure}

Finally, we show the performance of \bie{} (Algorithm ~\ref{alg:expectation_known}) for varying levels of impairment. The number of arms in this experiment is $10$. Impairment is stochastic and is simulated using the absolute value normal distribution with means = $\{2,6,10,14\}$ and the standard deviation being proportional to the arm index. The fixed impairment parameter is $N = 20$ and the time horizon is $10000$. From Figure~\ref{fig:4}, we can observe that as the cumulative regret increases as $\expd$ is increased. It can also be shown that the regret of \bie{} grows as O($\sqrt{\expd}$) (omitted).
\begin{figure}
\centering
\includegraphics[width=.5\columnwidth]{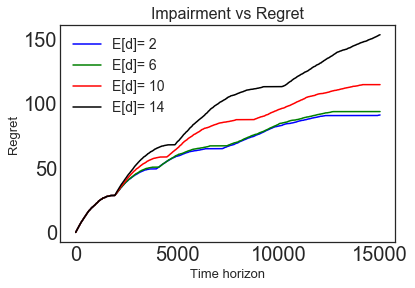}
\caption{Performance (cumulative regret) of \bie{} as the impairment level is varied.}
\label{fig:4}
\end{figure}

\section{Conclusion}\label{sec:conclusion}

In this work, we have addressed an interesting gap in the stochastic multi-armed bandits literature that models impairment, a phenomena where an algorithm only accrues reward if it has tried an arm multiple times in the recent past. This temporal dependence in collecting rewards is markedly different from extant literature on modeling delayed rewards. Our model of impairment captures both endogenous constraints (such as the algorithm being able to collect a reward only on multiple tries) and exogenous constraints (such as in distributed settings where only a subset of arms can be played for extended periods of time). Under this model, we develop new algorithms that either rotate between feasible arms in a round-robin fashion or play the same arm consecutively and show how their performance is impacted by impairment. Crucial to our algorithm designs is the concept of bucketing, that not only allows a further tradeoff in exploration and exploitation, but also allows us to meet arbitrary impairment restrictions. It also has uses in distributed online learning settings. Some future directions include proving tight lower bounds for the problem class, generalizing impairment to depend on past actions as well as rewards, and extending the bucketing strategy to the scenario where expected impairment is known.

\bibliographystyle{plainnat} 
\bibliography{impaired}

\begin{thebibliography}{43}
\providecommand{\natexlab}[1]{#1}
\providecommand{\url}[1]{\texttt{#1}}
\expandafter\ifx\csname urlstyle\endcsname\relax
  \providecommand{\doi}[1]{doi: #1}\else
  \providecommand{\doi}{doi: \begingroup \urlstyle{rm}\Url}\fi

\bibitem[Agrawal et~al.(1988)Agrawal, Hedge, and
  Teneketzis]{agrawal1988asymptotically}
Rajeev Agrawal, MV~Hedge, and Demosthenis Teneketzis.
\newblock Asymptotically efficient adaptive allocation rules for the multiarmed
  bandit problem with switching cost.
\newblock \emph{IEEE Transactions on Automatic Control}, 33\penalty0
  (10):\penalty0 899--906, 1988.

\bibitem[Audibert and Bubeck(2009)]{audibert2009minimax}
Jean-Yves Audibert and S{\'e}bastien Bubeck.
\newblock Minimax policies for adversarial and stochastic bandits.
\newblock In \emph{Conference on Learning Theory}, pages 217--226, 2009.

\bibitem[Audibert and Bubeck(2010)]{audibert2010best}
Jean-Yves Audibert and S{\'e}bastien Bubeck.
\newblock Best arm identification in multi-armed bandits.
\newblock In \emph{Conference on Learning Theory}, pages 13--p, 2010.

\bibitem[Auer and Ortner(2010)]{auer2010ucb}
Peter Auer and Ronald Ortner.
\newblock {UCB revisited: Improved regret bounds for the stochastic multi-armed
  bandit problem}.
\newblock \emph{Periodica Mathematica Hungarica}, 61\penalty0 (1-2):\penalty0
  55--65, 2010.

\bibitem[Auer et~al.(2002)Auer, Cesa-Bianchi, and Fischer]{auer2002finite}
Peter Auer, Nicolo Cesa-Bianchi, and Paul Fischer.
\newblock Finite-time analysis of the multiarmed bandit problem.
\newblock \emph{Machine Learning}, 47\penalty0 (2-3):\penalty0 235--256, 2002.

\bibitem[Banks and Sundaram(1994)]{banks1994switching}
Jeffrey~S Banks and Rangarajan~K Sundaram.
\newblock {Switching costs and the Gittins index}.
\newblock \emph{Econometrica}, pages 687--694, 1994.

\bibitem[Bosse et~al.(2015)Bosse, Mohr, Buss, Krautter, Weyrich, Herzog,
  J{\"u}nger, and Nikendei]{bosse2015benefit}
Hans~Martin Bosse, Jonathan Mohr, Beate Buss, Markus Krautter, Peter Weyrich,
  Wolfgang Herzog, Jana J{\"u}nger, and Christoph Nikendei.
\newblock The benefit of repetitive skills training and frequency of expert
  feedback in the early acquisition of procedural skills.
\newblock \emph{BMC medical education}, 15\penalty0 (1):\penalty0 22, 2015.

\bibitem[Buccapatnam et~al.(2014)Buccapatnam, Eryilmaz, and
  Shroff]{buccapatnam2014stochastic}
Swapna Buccapatnam, Atilla Eryilmaz, and Ness~B Shroff.
\newblock Stochastic bandits with side observations on networks.
\newblock \emph{ACM SIGMETRICS Performance Evaluation Review}, 42\penalty0
  (1):\penalty0 289--300, 2014.

\bibitem[Buccapatnam et~al.(2015)Buccapatnam, Tan, and
  Zhang]{buccapatnam2015information}
Swapna Buccapatnam, Jian Tan, and Li~Zhang.
\newblock Information sharing in distributed stochastic bandits.
\newblock In \emph{IEEE Conference on Computer Communications}, pages
  2605--2613. IEEE, 2015.

\bibitem[Bui et~al.(2011)Bui, Johari, and Mannor]{bui2011committing}
Loc~X Bui, Ramesh Johari, and Shie Mannor.
\newblock Committing bandits.
\newblock In \emph{Neural Information Processing Systems}, pages 1557--1565,
  2011.

\bibitem[Campbell and Keller(2003)]{campbell2003brand}
Margaret~C Campbell and Kevin~Lane Keller.
\newblock Brand familiarity and advertising repetition effects.
\newblock \emph{Journal of consumer research}, 30\penalty0 (2):\penalty0
  292--304, 2003.

\bibitem[Cesa-Bianchi et~al.(2018)Cesa-Bianchi, Gentile, and
  Mansour]{cesa2018nonstochastic}
Nicolo Cesa-Bianchi, Claudio Gentile, and Yishay Mansour.
\newblock Nonstochastic bandits with composite anonymous feedback.
\newblock In \emph{Conference on Learning Theory}, pages 750--773, 2018.

\bibitem[Chakrabarti et~al.(2009)Chakrabarti, Kumar, Radlinski, and
  Upfal]{chakrabarti2009mortal}
Deepayan Chakrabarti, Ravi Kumar, Filip Radlinski, and Eli Upfal.
\newblock Mortal multi-armed bandits.
\newblock In \emph{Neural Information Processing Systems}, pages 273--280,
  2009.

\bibitem[Chapelle and Li(2011)]{chapelle2011empirical}
Olivier Chapelle and Lihong Li.
\newblock {An empirical evaluation of Thompson sampling}.
\newblock In \emph{Neural Information Processing Systems}, pages 2249--2257,
  2011.

\bibitem[Dekel et~al.(2014)Dekel, Ding, Koren, and Peres]{dekel2014bandits}
Ofer Dekel, Jian Ding, Tomer Koren, and Yuval Peres.
\newblock {Bandits with switching costs: pow(T,2/3) regret}.
\newblock In \emph{ACM Symposium on Theory of Computing}, pages 459--467. ACM,
  2014.

\bibitem[DeKeyser(2007)]{dekeyser2007skill}
Robert DeKeyser.
\newblock Skill acquisition theory.
\newblock \emph{Theories in second language acquisition: An introduction},
  97113, 2007.

\bibitem[Even-Dar et~al.(2002)Even-Dar, Mannor, and Mansour]{even2002pac}
Eyal Even-Dar, Shie Mannor, and Yishay Mansour.
\newblock {PAC bounds for multi-armed bandit and Markov decision processes}.
\newblock In \emph{Conference on Learning Theory}, pages 255--270. Springer,
  2002.

\bibitem[Even-Dar et~al.(2006)Even-Dar, Mannor, and Mansour]{even2006action}
Eyal Even-Dar, Shie Mannor, and Yishay Mansour.
\newblock Action elimination and stopping conditions for the multi-armed bandit
  and reinforcement learning problems.
\newblock \emph{Journal of Machine Learning Research}, 7\penalty0 (1):\penalty0
  1079--1105, 2006.

\bibitem[Freedman(1975)]{freedman1975tail}
David~A Freedman.
\newblock On tail probabilities for martingales.
\newblock \emph{Annals of Probability}, pages 100--118, 1975.

\bibitem[Gajane et~al.(2017)Gajane, Urvoy, and Kaufmann]{gajane2017corrupt}
Pratik Gajane, Tanguy Urvoy, and Emilie Kaufmann.
\newblock Corrupt bandits for privacy preserving input.
\newblock \emph{ArXiv preprint arXiv:1708.05033}, 2017.

\bibitem[Gamarnik et~al.(2018)Gamarnik, Tsitsiklis, and
  Zubeldia]{gamarnik2018delay}
David Gamarnik, John~N Tsitsiklis, and Martin Zubeldia.
\newblock Delay, memory, and messaging tradeoffs in distributed service
  systems.
\newblock \emph{Stochastic Systems}, 8\penalty0 (1):\penalty0 45--74, 2018.

\bibitem[Garivier and Moulines(2008)]{garivier2008upper}
Aur{\'e}lien Garivier and Eric Moulines.
\newblock On upper-confidence bound policies for non-stationary bandit
  problems.
\newblock \emph{ArXiv preprint arXiv:0805.3415}, 2008.

\bibitem[Garivier and Moulines(2011)]{garivier2011upper}
Aur{\'e}lien Garivier and Eric Moulines.
\newblock On upper-confidence bound policies for switching bandit problems.
\newblock In \emph{Algorithmic Learning Theory}, pages 174--188. Springer,
  2011.

\bibitem[Gupta et~al.(2011)Gupta, Granmo, and Agrawala]{gupta2011successive}
Neha Gupta, Ole-Christoffer Granmo, and Ashok Agrawala.
\newblock Successive reduction of arms in multi-armed bandits.
\newblock In \emph{Research and Development in Intelligent Systems XXVIII},
  pages 181--194. Springer, 2011.

\bibitem[Hawkins et~al.(2009)Hawkins, Best, and Coney]{hawkins2009consumer}
Delbert Hawkins, Roger~J Best, and Kenneth~A Coney.
\newblock \emph{Consumer behavior}.
\newblock McGraw-Hill Publishing, 2009.

\bibitem[Hillel et~al.(2013)Hillel, Karnin, Koren, Lempel, and
  Somekh]{hillel2013distributed}
Eshcar Hillel, Zohar~S Karnin, Tomer Koren, Ronny Lempel, and Oren Somekh.
\newblock Distributed exploration in multi-armed bandits.
\newblock In \emph{Neural Information Processing Systems}, pages 854--862,
  2013.

\bibitem[Joulani et~al.(2013)Joulani, Gyorgy, and
  Szepesv{\'a}ri]{joulani2013online}
Pooria Joulani, Andras Gyorgy, and Csaba Szepesv{\'a}ri.
\newblock Online learning under delayed feedback.
\newblock In \emph{International Conference on Machine Learning}, pages
  1453--1461, 2013.

\bibitem[Kanade and Steinke(2014)]{kanade2014learning}
Varun Kanade and Thomas Steinke.
\newblock Learning hurdles for sleeping experts.
\newblock \emph{ACM Transactions on Computation Theory}, 6\penalty0
  (3):\penalty0 11, 2014.

\bibitem[Kang(2016)]{kang2016spaced}
Sean~HK Kang.
\newblock Spaced repetition promotes efficient and effective learning: Policy
  implications for instruction.
\newblock \emph{Policy Insights from the Behavioral and Brain Sciences},
  3\penalty0 (1):\penalty0 12--19, 2016.

\bibitem[Karnin et~al.(2013)Karnin, Koren, and Somekh]{karnin2013almost}
Zohar Karnin, Tomer Koren, and Oren Somekh.
\newblock Almost optimal exploration in multi-armed bandits.
\newblock In \emph{International Conference on Machine Learning}, pages
  1238--1246, 2013.

\bibitem[Kleinberg et~al.(2010)Kleinberg, Niculescu-Mizil, and
  Sharma]{kleinberg2010regret}
Robert Kleinberg, Alexandru Niculescu-Mizil, and Yogeshwer Sharma.
\newblock Regret bounds for sleeping experts and bandits.
\newblock \emph{Machine Learning}, 80\penalty0 (2-3):\penalty0 245--272, 2010.

\bibitem[Lykouris et~al.(2018)Lykouris, Mirrokni, and
  Paes~Leme]{lykouris2018stochastic}
Thodoris Lykouris, Vahab Mirrokni, and Renato Paes~Leme.
\newblock Stochastic bandits robust to adversarial corruptions.
\newblock In \emph{ACM SIGACT Symposium on Theory of Computing}, pages
  114--122. ACM, 2018.

\bibitem[Machleit and Wilson(1988)]{machleit1988emotional}
Karen~A Machleit and R~Dale Wilson.
\newblock Emotional feelings and attitude toward the advertisement: The roles
  of brand familarity and repetition.
\newblock \emph{Journal of Advertising}, 17\penalty0 (3):\penalty0 27--35,
  1988.

\bibitem[Maillard and Munos(2011)]{maillard2011adaptive}
Odalric-Ambrym Maillard and R{\'e}mi Munos.
\newblock {Adaptive bandits: Towards the best history-dependent strategy}.
\newblock In \emph{Artificial Intelligence and Statistics}, pages 570--578,
  2011.

\bibitem[Perchet et~al.(2016)Perchet, Rigollet, Chassang, Snowberg,
  et~al.]{perchet2016batched}
Vianney Perchet, Philippe Rigollet, Sylvain Chassang, Erik Snowberg, et~al.
\newblock Batched bandit problems.
\newblock \emph{Annals of Statistics}, 44\penalty0 (2):\penalty0 660--681,
  2016.

\bibitem[Pike-Burke et~al.(2018)Pike-Burke, Agrawal, Szepesvari, and
  Grunewalder]{pike2018bandits}
Ciara Pike-Burke, Shipra Agrawal, Csaba Szepesvari, and Steffen Grunewalder.
\newblock Bandits with delayed, aggregated anonymous feedback.
\newblock In \emph{International Conference on Machine Learning}, pages
  4102--4110, 2018.

\bibitem[Shah et~al.(2018)Shah, Blanchet, and Johari]{shah2018bandit}
Virag Shah, Jose Blanchet, and Ramesh Johari.
\newblock Bandit learning with positive externalities.
\newblock \emph{ArXiv preprint arXiv:1802.05693}, 2018.

\bibitem[Shivaswamy and Joachims(2012)]{shivaswamy2012multi}
Pannagadatta Shivaswamy and Thorsten Joachims.
\newblock Multi-armed bandit problems with history.
\newblock In \emph{Artificial Intelligence and Statistics}, pages 1046--1054,
  2012.

\bibitem[Solomon et~al.(2014)Solomon, Dahl, White, Zaichkowsky, and
  Polegato]{solomon2014consumer}
Michael~R Solomon, Dahren~William Dahl, Katherine White, Judith~L Zaichkowsky,
  and Rosemary Polegato.
\newblock \emph{{Consumer behavior: Buying, having, and being}}, volume~10.
\newblock Pearson London, 2014.

\bibitem[Swaminathan and Joachims(2015)]{swaminathan2015counterfactual}
Adith Swaminathan and Thorsten Joachims.
\newblock {Counterfactual risk minimization: Learning from logged bandit
  feedback}.
\newblock In \emph{International Conference on Machine Learning}, pages
  814--823, 2015.

\bibitem[Szita and Szepesv{\'a}ri(2011)]{szita2011agnostic}
Istv{\'a}n Szita and Csaba Szepesv{\'a}ri.
\newblock {Agnostic KWIK learning and efficient approximate reinforcement
  learning}.
\newblock In \emph{Conference on Learning Theory}, pages 739--772, 2011.

\bibitem[Valko et~al.(2014)Valko, Munos, Kveton, and
  Koc{\'a}k]{valko2014spectral}
Michal Valko, R{\'e}mi Munos, Branislav Kveton, and Tom{\'a}{\v{s}} Koc{\'a}k.
\newblock Spectral bandits for smooth graph functions.
\newblock In \emph{International Conference on Machine Learning}, pages 46--54,
  2014.

\bibitem[Xu and Yun(2018)]{xu2018reinforcement}
Kuang Xu and Se-Young Yun.
\newblock Reinforcement with fading memories.
\newblock In \emph{ACM Conference on Measurement and Modeling of Computer
  Systems}, pages 90--92. ACM, 2018.

\end{thebibliography}
\end{document}